%% file: main.tex
\documentclass[english, 11pt]{article}
\usepackage{mathpazo}
\usepackage[T1]{fontenc}
\usepackage[latin9]{inputenc}
\usepackage{geometry}
\usepackage{babel}
\usepackage{float}
\usepackage{url}
\usepackage{amsmath}
\usepackage{amsthm}
\usepackage{natbib}
\usepackage{amssymb}
\usepackage{undertilde}
\usepackage{algpseudocode}
\usepackage{xspace}
\usepackage{microtype}
\usepackage{enumitem}
\algrenewcommand\algorithmicindent{2.0em}%

\usepackage[unicode=true,pdfusetitle, bookmarks=true,bookmarksnumbered=false,bookmarksopen=false,breaklinks=true,pdfborder={0 0 1},backref=false,colorlinks=true, citecolor=blue] {hyperref}

\makeatletter

\floatstyle{ruled}
\newfloat{algorithm}{tbp}{loa}
\providecommand{\algorithmname}{Algorithm}
\floatname{algorithm}{\protect\algorithmname}

\theoremstyle{plain}
\newtheorem{prop}{\protect\propositionname}
\theoremstyle{plain}
\newtheorem{assumption}{\protect\assumptionname}
\theoremstyle{plain}
\newtheorem{lem}{\protect\lemmaname}
\theoremstyle{remark}
\newtheorem{rem}{\protect\remarkname}
\theoremstyle{plain}
\newtheorem{thm}{\protect\theoremname}
\theoremstyle{plain}
\newtheorem{cor}{\protect\corollaryname}

\usepackage{appendix,breakurl}
\PassOptionsToPackage{hyphens}{url}

\makeatother

\providecommand{\assumptionname}{Assumption}
\providecommand{\lemmaname}{Lemma}
\providecommand{\propositionname}{Proposition}
\providecommand{\remarkname}{Remark}
\providecommand{\theoremname}{Theorem}
\providecommand{\corollaryname}{Corollary}

\begin{document}

\global\long\def\pone{\mathsf{P1}}%
\global\long\def\ptwo{\mathsf{P2}}%
\global\long\def\mS{\mathcal{S}}%
\global\long\def\mA{\mathcal{A}}%
\global\long\def\mP{\mathcal{P}}%
\global\long\def\mQ{\mathcal{Q}}%
\global\long\def\mV{\mathcal{V}}%
\global\long\def\Pr{\mathbb{P}}%
\global\long\def\mP{\mathcal{P}}%
\global\long\def\E{\mathbb{E}}%
\global\long\def\R{\mathbb{R}}%
\global\long\def\det{\textup{det}}%
\global\long\def\tr{\textup{tr}}%
\global\long\def\simplex{\Delta}%
\global\long\def\dup{\textup{d}}%

\global\long\def\wover{\overline{w}}%
\global\long\def\wunder{\underline{w}}%
\global\long\def\Vover{\overline{V}}%
\global\long\def\Vunder{\underline{V}}%
\global\long\def\Qover{\overline{Q}}%
\global\long\def\Qunder{\underline{Q}}%
\global\long\def\Qotilde{\widetilde{Q}}%
\global\long\def\Qutilde{\utilde{Q}}%
\global\long\def\Votilde{\widetilde{V}}%
\global\long\def\Vutilde{\utilde{V}}%
\global\long\def\eps{\gamma}%
\global\long\def\epsover{\overline{\eps}}%
\global\long\def\epsunder{\underline{\eps}}%

\global\long\def\proj{\Pi}%
\global\long\def\projH{\proj_{H}}%
\global\long\def\sigmatilde{\widetilde{\sigma}}%

\global\long\def\findcce{\texttt{FIND\_CCE}}%
\global\long\def\findmax{\texttt{FIND\_MAX}}%
\global\long\def\findmin{\texttt{FIND\_MIN}}%
\global\long\def\sign{\textup{sign}}%

\newcommand{\br}{\ensuremath{\mathsf{br}}}
\newcommand{\OMVI}{OMNI-VI}
\newcommand{\omvi}{Optimistic Minimax Value Iteration}
\renewcommand{\baselinestretch}{1.1}

\let\hat\widehat
\let\tilde\widetilde
\let\check\widecheck

\title{Learning Zero-Sum Simultaneous-Move Markov Games Using Function Approximation and Correlated Equilibrium}

\author{Qiaomin Xie,$^\dagger$ Yudong Chen,$^\dagger$ Zhaoran Wang,$^\ddagger$ Zhuoran Yang$^\mathsection$ \footnote{Accepted for presentation at the Conference on Learning Theory (COLT) 2020. Emails: \texttt{qiaomin.xie@cornell.edu}, \texttt{yudong.chen@cornell.edu}, \texttt{zhaoranwang@gmail.com}, \texttt{zy6@princeton.edu}}\\ ~\\
	\normalsize $^\dagger$School of Operations Research and Information Engineering, Cornell University\\
	\normalsize $^\ddagger$Department of Industrial Engineering and Management Sciences, Northwestern University\\
	\normalsize $^\mathsection$Department of Operations Research and Financial Engineering, Princeton University
}

\date{}

\maketitle

\begin{abstract}
We develop provably efficient reinforcement learning algorithms for two-player zero-sum finite-horizon Markov games with simultaneous moves. 
To incorporate function approximation, we consider a family of Markov games where the reward function and
transition kernel possess a linear structure. Both the offline and online settings of the problems are considered. In the offline setting, we control both players and aim to find the Nash Equilibrium by minimizing the duality gap.  In the online setting, we control a single player playing against an arbitrary opponent and aim to minimize the regret. For both settings, we propose an optimistic variant of the least-squares minimax value iteration algorithm.
We show that our algorithm is computationally efficient and provably achieves an $\tilde O(\sqrt{d^3 H^3 T} )$ upper bound on the duality gap and regret, where $d$ is the linear dimension, $H$ the horizon and $T$ the total number of timesteps. Our results do not require additional assumptions on the sampling model. 

Our setting requires overcoming several new challenges that are absent in Markov decision processes or turn-based Markov games. In particular,  to  achieve optimism with simultaneous moves,  we construct both upper and lower confidence bounds of the value function, and then compute the optimistic policy by solving a general-sum matrix game with these bounds as the payoff matrices.
As finding the Nash Equilibrium of a general-sum game is computationally hard, our algorithm instead solves for a Coarse Correlated Equilibrium (CCE), which can be obtained efficiently. To our best knowledge, such a CCE-based scheme for  optimism has not appeared in the literature and might be of interest in its own right.

\end{abstract}

\input{intro.tex}

\input{background.tex}

\input{selfplay.tex}

\input{online.tex}

\input{proofthm.tex}

\input{conclusion.tex}

\section*{Acknowledgement}

Q.\ Xie and Y.\ Chen would like to thank Siddhartha Banerjee for inspiring discussion. Y.\ Chen is partially supported by NSF CRII award 1657420 and grant 1704828. 

\appendix
\appendixpage

\input{algos_turn.tex}

\input{techlemma.tex}

\input{appendix.tex}

\bibliographystyle{apalike}
\bibliography{../../rl_refs}

\end{document}

%% file: intro.tex
%!TEX root =main.tex

\section{Introduction}\label{sec:intro}

Reinforcement learning  \citep{sutton2018reinforcement} is typically modeled as a Markov Decision Process (MDP) \citep{puterman2014markov}, 
where an agent aims to learn the optimal decision-making rule via interaction with the environment. 
In Multi-agent reinforcement learning (MARL), several agents interact with each other and with the underlying environment, and their goal is to optimize their individual returns. This problem is often formulated under the framework of Markov games
 \citep{shapley1953stochastic}, which  is a generalization of the MDP model. 
Powered by function approximation techniques such as deep neural networks \citep{lecun2015deep, goodfellow2016deep},  
MARL has recently enjoyed tremendous empirical success across a variety of real-world applications. 
A partial list  of such applications includes the game of Go~\citep{silver2016alphago,silver2017alphagozero},
real-time strategy games~\citep{openaifive,vinyals2019alphastar}, Texas Hold'em poker~\citep{moravvcik2017deepstack, brown2018superhuman, brown2019superhuman}, autonomous driving~\citep{shalev2016safe_drive}, and learning communication and emergent behaviors \citep{foerster2016learning, lowe2017multi, bansal2017emergent, jaques2018social, baker2019emergent}; see the surveys in~\cite{busoniu2008marl_survey,zhang2019marl_overview}.

In contrast to the vibrant empirical study, theoretical understanding of MARL  
is relatively inadequate. Most existing work on Markov games assumes access to either a sampling oracle or a well-explored behavioral policy, which fails to capture the exploration-exploitation tradeoff that is fundamental in real-world applications of reinforcement learning. Moreover, these results mostly focus on the relatively simple turn-based setting.
An exception is the work in~\citet{wei2017online}, which extends the UCRL2 algorithm \citep{jaksch2010near} for MDP to zero-sum simultaneous-move Markov games. However, their approach  explicitly estimates the transition model and thus only works in the tabular setting. Problems with complicated state spaces and transitions necessitate the use of function approximation architectures. In this regard, a fundamental  question is left open:  Can we design a provably efficient  reinforcement learning  algorithm for Markov games under the function approximation setting?\\

In this paper, we provide an affirmative answer to this question for two-player zero-sum  Markov games with simultaneous moves and a linear structure.
In particular, we study an episodic setting, where  each episode consists of $H$ timesteps and  the players act simultaneously at each timestep. 
Upon reaching the $H$-th timestep, the episode terminates and players replay the game again by starting  a new episode. 
Here, the players have no knowledge of the system model (i.e., the transition kernel) nor access to a sampling oracle that returns the next state for an arbitrary state-action pair. Therefore, the players  have to learn the system from data by playing the game sequentially through each episode and repeatedly for multiple episodes.
More specifically, we study episodic Markov games under both the offline and online settings. In the offline setting,  both players are controlled by a central learner, and the goal is  to find an approximate Nash Equilibrium of the game, with the approximation error measured by a notion of duality gap.
In the online setting, we control one of the players  and play against an opponent who implements an arbitrary policy. Our goal is  to minimize the total regret, defined as the difference between the cumulative return of the controlled player and  its optimal achievable return when the opponent plays the best response policy. 
Both settings are generalizations of  the  regret minimization problem for MDPs. 
 
Furthermore, to incorporate function approximation, we consider Markov games with a linear structure, motivated by the linear MDP model recently studied in~\cite{jin2019linear}. In particular, we assume that both the transition kernel and the reward admit a $ d $-dimensional linear representation with respect to a known feature mapping, which can be potentially nonlinear  in its inputs. 
For both the online and offline settings, we propose the first provably efficient reinforcement learning algorithm without additional assumptions on the sampling model. Our algorithm is an Optimistic version of Minimax Value Iteration (\OMVI) with least squares estimation---a model-free approach---which   constructs  upper confidence bounds of the optimal action-value function to promote exploration. 
We show that the  \OMVI\ algorithm  is computationally efficient, and it provably achieves an $\widetilde O ( \sqrt{d^3 H^3 T})$ regret in the online setting and a similar duality gap guarantee in the offline setting, where  $ T $ is the total number of timesteps and $\widetilde O$ omits logarithmic terms.  Note that the bounds do not depend
on the cardinalities of the state and action spaces, which can be very large or even infinite.  When specialized to MDPs and linear bandits, our results can be compare with exiting regret bounds and are near-optimal. \\

We emphasize that the Markov game model poses several new and fundamental challenges that are absent in MDPs and arise due to subtle
game-theoretic considerations. Addressing these challenges require several new ideas, which we summarize as follows. 
\begin{enumerate}
 
\item \textbf{Optimism via  General-Sum Games.} In the offline simultaneous-move  setting, implementing the optimism principle for \emph{both} players amounts to constructing both upper and lower confidence bounds (UCB and LCB) for the optimal value function of the game. Doing so requires one to find, as an algorithmic subroutine, the solution of a \emph{general-sum}~(matrix) game where the two players' payoff functions correspond to the upper and lower bounds for the action-value (or Q) functions of the original Markov game, even though the latter is zero-sum to begin with. This stands in sharp contrast of turn-based games \citep{hansen2013strategy,jia2019feature,sidford2019solving}, in which each turn only involves constructing an UCB for one player.
 
\item \textbf{Using Correlated Equilibrium.}   
Finding the Nash equilibrium (NE) of a general-sum matrix game, however, is computationally hard in general
\citep{daskalakis2009complexity,chen2009settling}. Our second critical observation is that it suffices to find a \emph{Coarse Correlated Equilibrium (CCE)} \citep{moulin1978strategically,aumann1987correlated}  of the game. Originally developed in
algorithmic game theory, CCE is a tractable notion of equilibrium that strictly generalizes NE. In contrast to NE, a CCE can be found
efficiently in polynomial time even for general-sum games~\citep{papadimitriou2008computing,blum2008regret}.
Moreover, our analysis shows that using any CCE of the matrix general-sum game are sufficient for ensuring optimism for the original Markov game.  Thus, by using CCE instead of NE, we achieve efficient exploration-exploitation balance while preserving computational tractability. 

\item \textbf{Concentration and Game Stability.} The last challenge is more technical, arising in the analysis of the algorithm where we need to
establish certain uniform concentration bounds for the CCEs. As we elaborate later, the CCEs of a general-sum game are \emph{unstable}  (i.e., not Lipschitz) with respect to the payoff matrices. Therefore, standard approaches for proving uniform concentration, such as those based on covering/$\epsilon$-net
arguments, is fundamentally insufficient. We overcome this issue by carefully \emph{stabilizing} the algorithm, for which we make use of an $\epsilon$-net \emph{in the algorithm}. Moreover, we show that this can be done in a computationally efficient way via rounding on-the-fly, without explicitly maintaining the $ \epsilon $-net.
\end{enumerate}
We shall discuss the above challenges and ideas in greater details when we formally describe our algorithms. 
We note that our regret and duality gap bounds also imply polynomial sample complexity and PAC guarantees for learning the NEs of simultaneous-move Markov games. 
Moreover, as turn-based games can be viewed as a special case of simultaneous
games, where  at each state the reward and transition kernel only depend on the action of one of the players, our algorithms and guarantees readily apply to the turn-based setting.

\subsection{Related Work}

There is a large body of literature on applying reinforcement learning methods to Makove games (a.k.a.\ stochastic games). These results typically assume access to a sampling oracle, and most of them provide convergence guarantees that are asymptotic in nature.
In particular, under the tabular setting, 
the work in \cite{littman1996generalized, littman2001friend, littman2001value,greenwald2003correlated, hu2003nash, grau2018balancing} extends the value iteration and Q-learning algorithms~\citep{watkins1992q} to zero-sum and general-sum Markov games,  and that in \cite{perolat2018actor, srinivasan2018actor} 
extends  the   actor-critic  algorithm \citep{konda2000actor}.
Particularly related to us is the work in \cite{sidford2019solving}, which proposes a variance-reduced variant of the minimax Q-learning algorithm  with near-optimal sample complexity. We note that the theoretical results therein also require a sampling oracle, and they focus  on turn-based games, a special case of simultaneous-move games. 
The work in \citet{lagoudakis2012value, perolat2015approximate, perolat2016use, perolat2016softened,perolat2016learning,  yang2019theoretical} considers function approximation techniques  applied to variants of value-iteration methods and establishes finite-time convergence to the NEs of  two-player zero-sum  Markov games. Their results are based on the framework of fitted value-iteration \citep{munos2008finite} and the availability of a well-explored behavioral policy. The recent work  
\cite{jia2019feature} studies turn-based zero-sum Markov 
 games,  where the transition model is assumed to be embedded in some $ d $-dimensional feature space, extending the MDP model proposed by \citet{yang2019sample}.  
Assuming a sampling oracle, they propose a variant of Q-learning algorithm  that is guaranteed to find an
$\varepsilon$-optimal strategy using $\tilde{O}(  d\varepsilon^{-2}(1-\gamma)^{-4})$
samples, where $\gamma$ is a discount factor.  
In summary, all of the work above either assume a sampling oracle or a  well explored behavioral policy for drawing transitions, therefore effectively bypassing the exploration issue.

Our work builds on a line of research on provably efficient methods for MDPs without additional assumptions on the sampling model.
Most of the existing work focus on the tabular setting; see e.g., 
 \cite{strehl2006pac, jaksch2010near, osband2014generalization, azar2017minimax, dann2017unifying,agrawal2017optimistic,  jin2018q, russo2019worst,rosenberg2019online,  jin2019learning, zanette2019tighter, simchowitz2019non, dong2019q} and the references therein. 
Under the function approximation setting, sample-efficient algorithms have been proposed using linear function approximators
\citep{abbasi2019politex, abbasi2019exploration,  yang2019reinforcement,  du2019provably, wang2019optimism}, as well as nonlinear ones~\citep{wen2017efficient, jiang2017contextual, dann2018oracle,du2019provably,  dong2019sqrt, du2019provably2}. Among this line of work, our paper is most related to \citet{jin2019linear,  zanette2019frequentist, cai2019provably}, which consider linear MDP models and propose  optimistic and randomized variants of least-squares value iteration (LSVI) \citep{bradtke1996linear,  osband2014generalization} as well as optimistic variants of proximal policy optimization \citep{schulman2017proximal}.
Our linear Markov game model generalizes the MDP model considered in these papers, and our \OMVI\ algorithm can be viewed as a generalization of the optimistic LSVI method proposed in \citep{jin2019linear}. 
As mentioned before,  the game structures in our problem pose fundamental challenges that are absent in MDPs, and thus their algorithms cannot be trivially extended to our game setting. 
 
Work on provably sample efficient RL methods for Markov games is quite scarce. The only comparable work we are aware of is~\cite{wei2017online}, which proposes a model-based algorithm that extends the UCRL2 algorithm~\citep{jaksch2010near} for tabular MDPs to the game setting. Similarly to their work, we also consider both the online and offline settings and provide guarantees in terms of duality gap and regret. 
On the other hand, they only consider tabular setting, a special case of our linear model. Moreover,  their model-based algorithm explicitly estimates the Markov transition kernel and relies on the complicated technique of Extended Value Iteration, whose computational cost is quite high as it requires augmenting the state/action spaces. In comparison, our algorithm is model-free in the sense that it directly estimates the value functions; moreover, the computational cost of our algorithm only depends on the dimension $d$ of the feature and not the cardinality of the state space.

Finally, we remark that there is a line of work on robust MDPs~\citep{xu2012distributionally,lim2013robust_RL}, where an adversary chooses the transition kernel from an uncertainty set. This problem is closely related to our online setting, where the adversary chooses an action that determines the transition kernel. One technical difference is that in their setting, the uncertainty set is known yet the choice of the adversary is not directly observable, whereas in our case the adversary's action is observed but its influence on the transition and value functions needs to be estimated from data. The algorithms are also different: they take an model-based approach that finds the worst-case transition kernel from the uncertainty set, whereas our algorithm computes empirical estimates of the worst-case value functions using data. Also, their results apply only to the tabular setting of MDPs. 

%% file: background.tex
%!TEX root =main.tex

\section{Background and Preliminaries\label{sec:setup}}

In this section, we formally describe the setup for episodic two-player zero-sum Markov games with simultaneous moves. We then describe the setting for turn-based games, which can be viewed as a special case of simultaneous-moves games.

\subsection{Notation}\label{sec:notation}

For two quantities $ x $ and $y$ that potentially depend on the problem parameters ($ d, H, |\mA| $, $ T $, etc.), if $ x \ge C y $ holds for a universal absolute constant $ C>0 $, we write $ x\gtrsim y $, $ x=\Omega(y) $ and $ y=O(x) $.   For each real number $ u $, define the clipping operation $\projH(u)=\max\left\{ \min\left\{ u,H\right\} ,-H\right\} $. We use $ \|\cdot\| $ to denote the vector $ \ell_2 $ norm and $ \| \cdot \|_\textup{F} $ the matrix Frobenius norm. Given a positive semidefinite matrix $ A $, define the weighted $ \ell_2 $ norm $ \| v\|_A := \sqrt{v^\top A v}$ for the vector $ v$.

We sometimes need to consider a general-sum  \emph{matrix} (or \emph{normal form}) game with payoff matrices $ u_i  \in \R^{|\mA|\times |\mA|}, i\in\{1,2\} $ for two players denoted by $ \pone $ and $ \ptwo $.  If $ \pone $ and $ \ptwo $ take actions $a $ and $ b $, respectively, then $ \mathsf{P}i $ receives a payoff $ u_i(a,b) $. We use the convention that $ \pone $ tries to maximize the payoff and $ \ptwo $ tries to minimize. A joint distribution $\sigma \in\simplex(\mA\times\mA)$ of both players' actions is called a \emph{Coarse Correlated Equilibrium} \citep{moulin1978strategically,aumann1987correlated} of the game if it satisfies 
\begin{subequations}
	\label{eq:cce}
	\begin{equation}
		\E_{(a,b)\sim\sigma}\left[u_1(x,a,b)\right]  \ge\E_{b\sim\mP_{2}\sigma}\left[u_1(x,a',b)\right],\quad\forall a'\in\mA,\label{eq:cce1}
	\end{equation}
	\begin{equation}
		\E_{(a,b)\sim\sigma}\left[u_2(x,a,b)\right]  \le\E_{a\sim\mP_{1}\sigma}\left[u_2(x,a,b')\right],\quad\forall b'\in\mA,\label{eq:cce2}
	\end{equation}
\end{subequations}
where for $i\in\{1,2\}$, $\mP_{i}\sigma\in\simplex(\mA)$
denotes the $i$-th marginal of $\sigma$. In words, in a CCE the players choose their actions in a potentially correlated way such that no unilateral (unconditional) deviation from $ \sigma $ is beneficial.\footnote{We note in passing that there is a more restrictive notion of \emph{Correlated Equilibrium} (CE)~\citep{moulin1978strategically,aumann1987correlated}, in which the deviation is allowed to depend on the original actions. The set of CCEs include the set of CEs, which in turn includes the set of NEs. We use CCE in this paper as it is the easiest to compute among the three.} Note that a CCE $ \sigma = \sigma_1 \times \sigma_2 $ in product form is an NE.

\subsection{Simultaneous-Move Markov Games}\label{sec:setup_game}

 A two-player, zero-sum, simultaneous-moves, episodic Markov
game is defined by the  tuple 
\[
(\mS,\mA_{1},\mA_{2},r,\Pr,H),
\]
where $\mS$ is the state space, $\mA_{i}$ is a finite
set of actions that player $ i \in\{1,2\}$ can take, $r$ is reward function,
$\Pr$ is transition kernel and $H$ is the number of steps
in each episode. 
At each step $h\in[H]$, upon observing the state $ x \in\mS$, $ \pone $ and $ \ptwo $ take actions $a\in\mA_{1}$
and $b\in\mA_{2}$, respectively, and then both receive the reward $r_{h}(x,a,b)$.
The system then transitions to a new state $x' \sim \Pr_{h}(\cdot|x,a,b)$ according to the transition kernel.
Throughout this paper, we assume for simplicity that $ \mA_{1}=\mA_{2} = \mA $ and that the rewards $r_{h}(x,a,b)$ are deterministic functions of the tuple $(x,a,b)$ taking value in $[-1,1]$; generalization to the setting with  $ \mA_{1} \neq \mA_{2} $ and stochastic rewards is straightforward.

Denote by $\simplex\equiv\simplex(\mA)$ the probability simplex over the action space $\mA$. A stochastic policy of $ \pone  $ is a length-$H$ sequence of functions $\pi:=( \pi_{h}:\mS\rightarrow\simplex )_{h\in[H]}$. At each step $h\in[H]$ and state $x\in\mS$, $ \pone $ takes an action sampled 
from the distribution $\pi_{h}(x)$ over $\mA$. Similarly, a stochastic
policy of $ \ptwo $ is given by the sequence $\nu:=(\nu_{h}:\mS\rightarrow\simplex)_{h\in[H]}$.

\subsubsection{Value Functions\label{sec:bellman_simu}}

For a fixed pair of policies $ (\pi, \nu) $ for both players, the value and Q  (a.k.a.\ action-value) functions for the above game can
be defined in a manner analogous to the episodic Markov decision process
(MDP) setting:
\begin{align*}
V_{h}^{\pi,\nu}(x) :=\E\bigg [\sum_{t=h}^{H}r_{t}(x_{t},a_{t},b_{t})|x_{h}=x\bigg], \quad 
Q_{h}^{\pi,\nu}(x,a,b) :=\E\bigg[\sum_{t=h}^{H}r_{t}(x_{t},a_{t},b_{t})|x_{h}=x,a_{h}=a,b_{h}=b\bigg],
\end{align*}
where the expectation is over $a_{t}\sim\pi_t(x_{t}),$ $b_{t}\sim\nu_t(x_{t})$ and $x_{t+1}\sim\Pr_{t}(\cdot|x_{t},a_{t},b_{t})$. 
It is convenient to set $ V_{H+1}^{\pi,\nu}(x) \equiv Q_{H+1}^{\pi,\nu}(x) \equiv 0$ for the terminal reward.
Under the boundedness assumption on the reward,  it is easy see that all value functions are bounded:
\[
\left|V_{h}^{\pi,\nu}(x)\right|\le H\quad\text{and}\quad\left|Q_{h}^{\pi,\nu}(x,a,b)\right|\le H,\qquad\forall x,a,b,h,\pi,\nu.
\]

In the zero-sum setting, for a given initial state $ x_1 $, $\pone$ aims to maximize $ V_1^{\pi,\nu} (x_1)$ whereas  $\ptwo$ aims to minimize
it. Accordingly, we introduce the value and Q functions when $ \pone $ plays the best response to a fixed policy $ \nu $ of $ \ptwo $:
\begin{align*}
V_{h}^{*,\nu}(x) = \max_{\pi} V_{h}^{\pi,\nu}(x)
\quad\text{and}\quad
Q_{h}^{*,\nu}(x,a,b) = \max_{\pi} Q_{h}^{\pi,\nu}(x,a,b).
\end{align*}
Analogously, when $ \ptwo $ plays the best response to $ \pone $'s policy $ \pi $, we define
\begin{align*}
V_{h}^{\pi,*}(x) = \min_{\nu} V_{h}^{\pi,\nu}(x)
\quad\text{and}\quad
Q_{h}^{\pi,*}(x,a,b) = \min_{\nu} Q_{h}^{\pi,\nu}(x,a,b).
\end{align*}

A Nash Equilibrium (NE) of the game is a pair of stochastic policies  $ (\pi^*, \nu^*) $ that are the best response to each other; that is,
\begin{align}
\label{eq:NE}
V_1^{\pi^*,\nu^*} (x_1) = V_1^{*,\nu^*}(x_1) = V_1^{\pi^*,*}(x_1),\qquad x_1\in \mS.
\end{align}
We assume that the game satisfies appropriate regularity conditions so that an NE exists and their value is unique.\footnote{This holds, e.g., when the state space is compact~\citep{maitra1970stochastic,maitra1971stochastic2}.} 
Correspondingly, let $ V_h^* (x) := V_h^{\pi^*, \nu^*}(x)  $ and $ Q_h^* (x,a,b) := Q_h^{\pi^*, \nu^*}(x,a,b)  $ denote the values of the NE at step $ h $.

Define the following shorthand for conditional expectation for the step-$ h $ transition:
\[
[\Pr_{h}V](x,a,b):=\E_{x'\sim\Pr_{h}(\cdot|x,a,b)}[V(x')] = \int V(x') \dup \Pr_h(x' | x,a,b).
\]
While not explicitly needed in our analysis, we note that the value/Q functions for the NE satisfy the Bellman equations
\begin{subequations}
	\label{eq:bellman}
\begin{align}
&Q_{h}^{*}(x,a,b) =r_{h}(x,a,b)+(\Pr_{h}V_{h+1}^{*})(x,a,b), \label{eq:bellman1} \\
\quad\text{and}\quad&
V_{h}^{*}(x) =\max_{A\in\Delta}\min_{B\in\Delta}\E_{a\sim A,b\sim B}Q_{h}^{*}(x,a,b)=\min_{B\in\Delta}\max_{A\in\Delta}\E_{a\sim A,b\sim B}Q_{h}^{*}(x,a,b). \label{eq:bellman2} 
\end{align}
\end{subequations}
The fixed-policy and best-response value/Q functions, $V_h^{\pi,\nu}, V_h^{\pi,*},V_h^{*,\nu}, Q_h^{\pi,\nu}, Q_h^{\pi,*}$ and $Q_h^{*,\nu} $, satisfy a similar set of Bellman equations; we omit the details.

The following weak duality result, which follows immediately from definition, relates the above value and Q functions.
\begin{prop}[Weak Duality]
\label{prop:weak_duality_simu}For each policy pair $(\pi,\nu)$
and each $h\in[H]$, $(x,a,b)\in\mS\times \mA \times \mA$, we have 
\begin{align*}
Q_{h}^{\pi,*}(x,a,b)\leq& Q_{h}^{*}(x,a,b)\leq Q_{h}^{*,\nu}(x,a,b),&
V_{h}^{\pi,*}(x)\leq& V_{h}^{*}(x)\leq V_{h}^{*,\nu}(x),\\
Q_{h}^{\pi,*}(x,a,b)\leq& Q_{h}^{\pi,\nu}(x,a,b)\leq Q_{h}^{*,\nu}(x,a,b),&
V_{h}^{\pi,*}(x)\leq& V_{h}^{\pi,\nu}(x)\leq V_{h}^{*,\nu}(x).
\end{align*}
\end{prop}

\subsubsection{Linear Structures\label{sec:setup_linear}}

We assume that  both the reward function and transition
kernel have a linear structure. 
\begin{assumption}[Linearity and Boundedness]
\label{assu:linear_bounded_simu} For each $ (x,a,b)\in\mS \times \mA \times \mA $ and $h\in[H]$, we have
\begin{align*}
r_{h}(x,a,b) =\phi(x,a,b)^{\top}\theta_{h}
\qquad\text{and}\qquad
\Pr_{h}(\cdot |x,a,b)  =\phi(x,a,b)^{\top}\mu_{h}(\cdot),
\end{align*}
where $\phi:\mS\times\mA\times\mA\rightarrow\R^{d}$ is a known feature map, $\theta_{h}\in\R^{d}$ is an unknown vector and $\mu_h = \big(\mu_{h}^{(i)} \big)_{i\in[d]}$ is a vector of  $d$ unknown (signed) measures on $\mS$. We assume that $\left\Vert \phi(\cdot,\cdot,\cdot)\right\Vert \le1$, $\left\Vert \theta_{h}\right\Vert \le\sqrt{d}$ and 
$\left\Vert \mu_{h}(\mS)\right\Vert \le\sqrt{d}$  for all $h\in[H]$, where $ \|\cdot\| $ is the vector $ \ell_2 $ norm.
\end{assumption}

Note that boundedness of the linear weights $\theta_h$ and $\mu_h$ allows for certain covering and concentration arguments in the analysis; also see~\cite[Section 2.1]{jin2019linear} for a discussion on the specific choice of normalization above. It is also easy to see that the linearity assumption above implies that the Q functions are linear.
\begin{lem}[Linearity of Value Function]
\label{lem:linearity_simu}
Under Assumption~\ref{assu:linear_bounded_simu}, for any policy pair $(\pi,\nu)$ and any
$h\in[H]$, there exists a vector $w_{h}^{\pi,\nu}\in\R^{d}$ such
that 
\[
Q_{h}^{\pi,\nu}(x,a,b)=\left\langle \phi(x,a,b),w_{h}^{\pi,\nu}\right\rangle ,\qquad\forall(x,a,b)\in\mS\times\mA\times\mA.
\]
\end{lem}

\begin{proof}
By Bellman equation and linearity of $r_{h}$ and $\Pr_{h}$, we have
\begin{align*}
Q_{h}^{\pi,\nu}(x,a,b) & =r_{h}(x,a,b)+\Pr_{h}V_{h+1}^{\pi,\nu}(x,a,b) =\phi(x,a,b)^{\top}\theta_{h}+\int V_{h+1}^{\pi,\nu}(x')\phi(x,a,b)^{\top}\dup\mu_{h}(x').
\end{align*}
Letting $w_{h}^{\pi,\nu}:=\theta_{h}+\int V_{h+1}^{\pi,\nu}(x')\textup{d}\mu_{h}(x')$
proves the lemma.
\end{proof}

\begin{rem}
\label{rem:linear_best_response}
Since $Q_{h}^{\pi,*}(x,a,b)=Q_{h}^{\pi,\br(\pi)}(x,a,b)$,
where $\br(\pi) \in \arg\min_{\nu} Q_h^{\pi,\nu}(x,a,b)$ is the best response policy to $\pi$, it follows
immediately from Lemma~\ref{lem:linearity_simu} that $Q_{h}^{\pi,*}(x,a,b)=\left\langle \phi(x,a,b),w_{h}^{\pi,*}\right\rangle $
for some $w_{h}^{\pi,*}\in\R^{d}$. Similarly, we have $Q_{h}^{*,\nu}(x,a,b)=\left\langle \phi(x,a,b),w_{h}^{*,\nu}\right\rangle $
for some $w_{h}^{*,\nu}\in\R^{d}$.
\end{rem}

The linear setting above covers the \emph{tabular setting} as a special case, where $ d=|\mS|\cdot|\mA|^{2} $ and  $\phi(x,a,b)$ is the indicator vector for the tuple $(x,a,b)$. It is also clear that MDPs are a special case of Markov games when $\ptwo$ plays a fixed and known policy. In particular, our setting covers both tabular MDPs as well as the linear MDP setting considered in the work~\cite{jin2019linear}.
Finally, as we elaborate in Section~\ref{sec:setup_turn} to follow, turn-based Markov Games can also be viewed as a special case of our setting.

\begin{rem}
	\label{rem:linear_necessary}
	Linearity of the reward and transition kernel is a \emph{strictly} stronger assumption than linearity of the value functions. Our analysis makes crucial use of this stronger assumption, which ensures that the linearity of value functions is preserved under the Bellman equation. In fact, it is likely that this assumption is essential for developing efficient algorithms, in view of recent hardness result in~\cite{du2019representation} that only assumes near-linearity of value functions of MDPs (a special case of Markov games).
\end{rem}

\subsection{Turn-Based Markov games\label{sec:setup_turn}}

In turn-based games, at each state only one player takes an action. Without loss of generality, we may partition the state space as $\mS=\mS_{1}\cup\mS_{2},$
where $\mS_{i}$ are the states at which it is player $ i $'s turn to play.\footnote{The assumption $\mS_{1}\cap\mS_{2}=\emptyset$ is satisfied if one incorporates the ``turn'' of the player as part of the state.} For each state $x\in\mS,$ let $I(x)\in\{1,2\}$ indicate the current player to play,  so that $x\in\mathcal{S}_{I(x)}$. At each step
$h\in[H]$, player $I(x)$ observes the current state $ x $ and takes an action $ a $; then the two players receive the reward $r_{h}(x,a)$, and the system transitions to a new state $x' \sim \Pr_{h}(\cdot|x,a)$.  

The value/Q functions $ V^{\pi, \nu}_h(x), Q^{\pi, \nu}_h(x,a) $ etc., as well as the corresponding NE of the game, can be defined in a completely analogous way as in the simultaneous-move setting. Similarly to Assumption~\ref{assu:linear_bounded_simu}, we also assume that the game has a linear structure.
\begin{assumption}[Linearity and Boundedness, Turn-Based]
	\label{assu:linear_bounded_turn} For each $ (x,a) \in \mS \times \mA $ and $h\in[H]$, we have
	\begin{align*}
	r_{h}(x,a)  =\phi(x,a)^{\top}\theta_{h}
	\qquad\text{and}\qquad
	\Pr_{h}(\cdot|x,a)  =\phi(x,a)^{\top}\mu_{h}(\cdot),
	\end{align*}
	where $\phi:\mS\times\mA\rightarrow\R^{d}$ is a known feature map, $\theta_{h}\in\R^{d}$ is an unknown vector and $\{\mu_{h}^{(i)}\}_{i\in[d]}$
	are $d$ unknown (signed) measures on $\mS$. We assume that $\left\Vert \phi(\cdot,\cdot)\right\Vert \le1$,
	$\left\Vert \mu_{h}(\mS)\right\Vert \le\sqrt{d}$ and $\left\Vert \theta_{h}\right\Vert \le\sqrt{d}$ for all $h\in[H]$.
\end{assumption}

One may view a turn-based game as a special case of a simultaneous-move
game, where at each state only one of the players is ``active''
and the other player's action has no influence on the reward or the
transition. Formally, for each $x\in S_{1}$, the values of $r_{h}(x,a,b)$,
$\Pr_{h}(\cdot|x,a,b)$ and $\phi(x,a,b)$ are independent of $b$; for each $x\in\mS_{2}$,
they are independent of $a$.

%% file: selfplay.tex
%!TEX root =main.tex
\section{Main Results for the Offline Setting\label{sec:main_simu}}

In this section, we consider the offline setting, where a central controller controls 
both players. The goal of the controller is learn a Nash equilibrium $ (\pi^*, \nu^*) $ of the game in episodic setting. 
In what follows, we formally define the problem setup and objectives, and then present our algorithm and provide theoretic guarantees for its performance.

\subsection{Setup and Performance Metrics}\label{sec:setup_offline}

In the episodic setting, the Markov game is played for $ K $ episodes, each of which consists of $ H $ timesteps. 
At the beginning of the $k$-th episode, an arbitrary initial state
$x_{1}^{k}$ is chosen. Then the players  $\pone$ and $ \ptwo $ play according to the policies $ \pi^k =(\pi^k_h)_{h\in[H]} $ and $\nu^{k}=(\nu^k_h)_{h\in[H]}$, respectively, which may adapt to  observations from past episodes. The game terminates after $ H $ timesteps and restarts for the $ (k+1) $-th episode. Note that expected reward for $\pone $ and $ \ptwo $ in the $ k $-th episode is $ V^{\pi^k,\nu^k}_1(x_1^k) $.

\paragraph{Duality gap guarantees:}

Recall the weak duality property in Proposition~\ref{prop:weak_duality_simu}, which states the value of the NE, $ V^*_1(x_1) $, is sandwiched between $ V_1^{\pi^k,*}(x_1) $ and $ V_1^{*,\nu^k} (x_1) $. Therefore, it is natural to use the duality gap $ V_1^{*,\nu^k} (x_1)-V_1^{\pi^k,*}(x_1) $ to measure how well the policy $ (\pi^k,\nu^k) $ in the $ k $-th episode  approximates the NE.
Accordingly, we aim to bound the following total duality gap:
\begin{equation}\label{eq:gap}
\mathrm{Gap}(K):=\sum_{k=1}^{K}\left[V_{1}^{*,\nu^{k}}(x_{1}^{k})-V_{1}^{\pi^{k},*}(x_{1}^{k})\right].
\end{equation}
Another way to interpret the above objective is as follows. Define the \emph{exploitability}~\citep{davis2014using} of $ \pone $ and $ \ptwo $, respectively, as
\begin{align*} 
\mathrm{Exploit}_1(\pi^k, \nu^k)  :=  V_1^{\pi^k,\nu^k} (x_1^k)  -  V_1^{\pi^k,  *} (x_1^k) \quad\text{and}\quad
\mathrm{Exploit}_2(\pi^k, \nu^k)  :=  V_1^{*,\nu^k} (x_1^k)  -  V_1^{\pi^k, \nu^k} (x_1^k), 
\end{align*}
both of which are nonnegative by Proposition~\ref{prop:weak_duality_simu}.
Here $ \mathrm{Exploit}_i(\pi^k, \nu^k)  $ measures the potential loss of player $ i\in\{1,2\} $ in the $ k $-th episode if the other player  unilaterally switched to the best response policy. The total duality gap can then be rewritten  as 
\begin{align*} 
\mathrm{Gap}(K) =\sum_{k=1}^{K} \Bigl[\mathrm{Exploit}_1 (\pi^k, \nu^k )  + \mathrm{Exploit}_2 (\pi^k, \nu^k ) \Bigr ],
\end{align*}
which is the sum of the exploitability of both players accumulated over $ K $ episodes. Also note that in special cases of MDPs, $ \mathrm{Gap}(K) $ reduces to the usual notion of total regret.

\paragraph{Sample complexity and PAC guarantees:}

Another performance metric is the sample complexity for finding an approximate NE. In particular, suppose that for all episodes the initial states $x_{1}$ are sampled
from the same fixed distribution. We are interested in the
number of episodes $K$ (or equivalently the number of samples $T=KH$)
needed to find a policy pair $(\pi,\nu)$ satisfying
\[
V_{1}^{*,\nu}(x_{1})-V_{1}^{\pi,*}(x_{1})\le\epsilon \qquad\text{with probability at least \ensuremath{1-\delta}}.
\]
In light of Proposition~\ref{prop:weak_duality_simu}, the above inequality implies that $ (\pi,\nu) $ is an $ \epsilon $-approximate NE in the sense that 
\begin{align*}
V_1^{*,\nu} (x_1) - \epsilon \le 
V_1^{\pi,\nu} (x_1) \le V_1^{\pi,*} (x_1) + \epsilon;
\end{align*}
that is, $ (\pi,\nu) $ satisfies the definition~\eqref{eq:NE} of NE up to an $ \epsilon $ error.
As we discuss in details after presenting our main theorem, a bound
on the total duality gap implies a bound on the sample complexity. Such a bound in turn implies a PAC-type guarantee in the sense of Kakade~\citep{kakade2003thesis}, which stipulates that an $ \epsilon $-approximate NE is played in all but a small number of timesteps.

\subsection{Algorithm}\label{sec:alg_offline}

We now present our algorithm, \omvi\ (\OMVI) with least squares estimation, which is given
as Algorithm~\ref{alg:simu}.

\begin{algorithm}[tbh]
	\caption{\omvi\ (Simultaneous Move, Offline)\label{alg:simu}}
\begin{algorithmic}[1]
	
\State \textbf{Input:} bonus parameter $ \beta>0 $.

\For {episode $k=1,2,\ldots,K$}

\State Receive initial state $x_{1}^{k}$

\For {step $h=H,H-1,\ldots,2,1$} \Comment{update policy}

\State $\Lambda_{h}^{k}\leftarrow\sum_{\tau=1}^{k-1}\phi(x_{h}^{\tau},a_{h}^{\tau},b_{h}^{\tau})\phi(x_{h}^{\tau},a_{h}^{\tau},b_{h}^{\tau})^{\top}+I.$  %\Comment{least-squares estimation}

\State $\wover_{h}^{k}\leftarrow(\Lambda_{h}^{k})^{-1}\sum_{\tau=1}^{k-1}\phi(x_{h}^{\tau},a_{h}^{\tau},b_{h}^{\tau})\left[r_{h}(x_{h}^{\tau},a_{h}^{\tau},b_{h}^{\tau})+\Vover_{h+1}^{k}(x_{h+1}^{\tau})\right]$.

\State $\wunder_{h}^{k}\leftarrow(\Lambda_{h}^{k})^{-1}\sum_{\tau=1}^{k-1}\phi(x_{h}^{\tau},a_{h}^{\tau},b_{h}^{\tau})\left[r_{h}(x_{h}^{\tau},a_{h}^{\tau},b_{h}^{\tau})+\Vunder_{h+1}^{k}(x_{h+1}^{\tau})\right]$.

\State $\Qover_{h}^{k}(\cdot,\cdot,\cdot)\leftarrow\projH\left\{ (\wover_{h}^{k})^{\top}\phi(\cdot,\cdot,\cdot)+\beta\sqrt{\phi(\cdot,\cdot,\cdot)^{\top}(\Lambda_{h}^{k})^{-1}\phi(\cdot,\cdot,\cdot)}\right\} .$ %\Comment{UCB/LCB for Q function}

\State $\Qunder_{h}^{k}(\cdot,\cdot,\cdot)\leftarrow\projH\left\{ (\wunder_{h}^{k})^{\top}\phi(\cdot,\cdot,\cdot)-\beta\sqrt{\phi(\cdot,\cdot,\cdot)^{\top}(\Lambda_{h}^{k})^{-1}\phi(\cdot,\cdot,\cdot)}\right\} .$

\State For each $x$, let $\sigma_{h}^{k}(x)\leftarrow\findcce\left(\Qover_{h}^{k},\Qunder_{h}^{k},x\right).$ %\Comment{UCB/LCB for value function}

\State $\Vover_{h}^{k}(x)\leftarrow\E_{(a,b)\sim\sigma_{h}^{k}(x)}\Qover_{h}^{k}(x,a,b)$ for each $x$.

\State $\Vunder_{h}^{k}(x)\leftarrow\E_{(a,b)\sim\sigma_{h}^{k}(x)}\Qunder_{h}^{k}\left(x,a,b\right)$
for each $x$.

\EndFor

\For {step $h=1,2,\ldots,H$} \Comment{execute policy}

\State  Sample $(a_{h}^{k},b_{h}^{k})\sim\sigma_{h}^{k}(x_{h}^{k})$.

\State  $\pone$ takes action $a_{h}^{k}$; $\ptwo$ takes
action $b_{h}^{k}$.

\State  Observe next state $x_{h+1}^{k}$.

\EndFor

\EndFor

\end{algorithmic}
\end{algorithm}

 In each episode $ k $, the algorithm first constructs the policies for both players (lines 4--13), and then executes the policy to play the game (lines 14--18).
The construction of the policy is done through backward induction with respect to the timestep $ h $. In each timestep, we first compute upper/lower estimates $ \wover_{h},\wunder_{h} \in \R^d$ of the linear coefficients of the  Q-function. This is done by approximately solving the Bellman equation~\eqref{eq:bellman} using (regularized) least-squares estimation, for which we use empirical data from the previous $ k-1 $ episodes to estimate the unknown transition kernel $ \Pr_h $ (lines 5--7). Then, to encourage exploration, we construct UCB/LCB for the Q function by adding/subtracting an appropriate bonus term (lines 8--9). The bonus takes the form $ \beta\sqrt{\phi^\top (\Lambda_{h}^{k})^{-1}\phi} $, where  $\Lambda_{h}^{k}$ is the regularized Gram matrix defined in line 5 of the algorithm. This form of bonus is common in the literature of linear bandits~\citep{lattimore2018bandit}. The next and crucial step is to convert the UCB/LCB $  (\Qover_{h}, \Qunder_{h}) $ for the Q function into UCB/LCB $ (\Vover_{h}, \Vunder_{h}) $ for the \emph{value} function (lines 10--12). This step turns out to be quite delicate; we elaborate below.

Note that $\Vover_{h}(x)$ and $\Vunder_{h}(x)$ should correspond to the actions $(a',b')$ that would be actually played at state $x$, that is, $\Vover_{h}(x) = \Qover_{h}(x,a',b')$ (in expectation w.r.t.\ randomness of the stochastic policy; similarly for $\Vunder_h(x)$), so that these upper/lower bounds can be tightened up using empirical observations from these actions. To construct these bounds, one may be tempted to let each player \emph{independently} compute the maximin or minimax values and actions. That is, one may let $ \pone $ play the action $ a' = \arg\max_a\min_b \Qover_{h}^k(x,a,b) $ and $ \ptwo $ play $ b' = \arg\min_b\max_a \Qunder_{h}^k(x,a,b) $, and then set $ \Vover_{h}^k(x) \leftarrow  \Qover_{h}^k (x,a', b')$ and $ \Vunder_{h}^k(x) \leftarrow  \Qunder_{h}^k (x,a',b')$. Unfortunately, such a $ \Vover_{h}^k(x) $ is \emph{not} a valid upper bound for  the true value, since $ \Qover_{h}^k \neq \Qunder_h^k $ in general and hence $ \Qover_{h}^k (x,a',b') \neq \max_a\min_b \Qover_{h}^k(x,a,b)$.

Instead, we must \emph{coordinate} both players for their choices of actions, which is done by solving the \emph{general-sum} matrix game with payoff matrices $ \Qover_{h}^k(x,\cdot,\cdot) $ and $ \Qunder_{h}^k(x,\cdot,\cdot) $. Finding the NE for general-sum games gives valid UCB/LCB, but doing so is computationally intractable \citep{daskalakis2009complexity,chen2009settling}. Fortunately, computing an (approximate) CCE of the matrix game turns out to be sufficient as well.  For technical reasons elaborated in the next subsection, the subroutine $\findcce$ for finding the CCE is implemented in a specific way as follows. Let $\mQ$ be
the class of functions $Q:\mS\times\mA\times\mA\to\R$ with the parametric
form
\begin{equation}
Q(x,a,b)=\projH\left\{ \left\langle w,\phi(x,a,b)\right\rangle +\rho\beta\sqrt{\phi(x,a,b)^{\top}A\phi(x,a,b)}\right\} ,\label{eq:Qclass}
\end{equation}
where the parameters $(w,A,\rho)\in \R^d\times \R^{d\times d}\times \{\pm 1\}$ satisfy  $\left\Vert w\right\Vert \le2H\sqrt{dk}$ and 
$\left\Vert A\right\Vert _{F}\le\beta^{2}\sqrt{d}$.
Let $\mQ_{\epsilon}$ be a fixed $\epsilon$-covering of $\mQ$ with
respect to the $\ell_{\infty}$ norm $\left\Vert Q-Q'\right\Vert _{\infty}:=\sup_{x,a,b}\left|Q(x,a,b)-Q'(x,a,b)\right|$.
With these notations, we present the subroutine $\findcce$ in Algorithm~\ref{alg:find_cce}. 
The algorithm effectively ``rounds'' the game $ \big(\Qover_{h}^k(x,\cdot,\cdot), \Qunder_{h}^k(x,\cdot,\cdot) \big)$ of interest into a nearby game in the finite $\epsilon $-cover  $\mQ_{\epsilon}\times\mQ_{\epsilon}$, and then uses the CCE of the latter game as an surrogate of the CCE of the original game.
We remark that this rounding step can be implemented efficiently without explicitly computing/maintaining the (exponentially large) $ \epsilon $-net; see Appendix~\ref{sec:find_cce_implement} for details.

\begin{algorithm}[tbh]
\caption{$\protect\findcce$\label{alg:find_cce}}

\begin{algorithmic}[1]
\State \textbf{Input:} $\Qover_{h}^{k}$, $\Qunder_{h}^{k}$, $x$ and discretization parameter $ \epsilon>0 $.

\State Pick a pair $\left(\Qotilde,\Qutilde\right)$ in $\mQ_{\epsilon}\times\mQ_{\epsilon}$
satisfying $\left\Vert \Qotilde-\Qover_{h}^{k}\right\Vert _{\infty}\le\epsilon$
and $\left\Vert \Qutilde-\Qunder_{h}^{k}\right\Vert _{\infty}\le\epsilon$.

\State For the input $x$, let $\widetilde{\sigma}(x)$ be the CCE (cf.\ equation~(\ref{eq:cce}))
of the matrix game with payoff matrices
\[
\Qotilde(x,\cdot,\cdot)\text{ for }\pone\quad\text{and}\quad\Qutilde(x,\cdot,\cdot)\text{ for }\ptwo.
\]

\State \textbf{Output:} $\sigmatilde(x)$.

\end{algorithmic}
\end{algorithm}

\subsubsection{Technical Considerations for $ \findcce $}\label{sec:tech_consideration}

We explain the motivation for using rounding and an $ \epsilon $-cover in $ \findcce $.
First,  note that  the least-squares step of Algorithm~\ref{alg:simu} (line 5--7) uses data from all previous episodes. This introduces complicated probabilistic dependency between the estimation target $ \Vover_{h+1}^k $ and the linear features $ \phi(x_h^\tau, a_h^\tau, b_h^\tau), \tau\in[k-1] $, as they both depend on past data. Such dependency is not present in the usual least-squares estimation in supervised learning. To overcome this issue, a standard approach is to use a covering argument to establish uniform concentration bounds valid for all value functions $ \Vover_{h+1}^k $.\footnote{In the \emph{tabular} setting, recent work in~\cite{agarwal2019sparse,ding2018loo,pananjady2019mrp} bypasses the use of uniform concentration by employing sophisticated \emph{leave-one-out} techniques to decouple the probabilistic dependency. However, it is unclear how such techniques can be used in the function approximation setting.}

While it is straightforward to construct a cover for the Q functions (as we have done in $ \findcce $), doing so for the value functions is challenging  due to \emph{instability} of the equilibria of general-sum games.  In particular, recall that the value function is defined by the CCE value of a general-sum game with two payoff matrices given by the Q functions. The CCE value, however, is \emph{not} a Lipschitz function of the payoff matrices, hence a cover for the former does \emph{not} follow from a cover for the latter.
Indeed, suppose that a game has payoff matrices $ (\Qover_{h}^k,\Qunder_{h}^k)  $ that are $ \epsilon $-close to  another game $( \Qotilde, \Qutilde)$ from the cover $\mQ_{\epsilon}\times\mQ_{\epsilon}$.  Lemma~\ref{lem:not_lipschitz} in Appendix~\ref{sec:stability} shows the following:
\begin{enumerate}[leftmargin=.4in, topsep=4pt] 
\item[(i)] The CCE \emph{values} of the above two games may be $ 1+\epsilon $ away from each other.
\end{enumerate}
Interestingly, general-sum matrix games satisfy another property, proved in Lemma~\ref{lem:eps_cce}, that is seemingly contradictory to property (i) above:
\begin{enumerate}[leftmargin=.4in,topsep=4pt] 
\item[(ii)] The CCE \emph{policy} of the game $( \Qotilde, \Qutilde)$ from the cover is a $ 2\epsilon $-approximate CCE policy for the original game $ (\Qover_{h}^k,\Qunder_{h}^k)  $, and vice versa. 
\end{enumerate}
Here a $ 2\epsilon $-approximate CCE policy is one that satisfies the definition~\eqref{eq:cce} of CCE with an additive $ 2\epsilon $ error on the RHS.
The proof of Lemma~\ref{lem:not_lipschitz} gives an example in which properties (i) and (ii) hold simultaneously.

Due to property (i) above, it is unclear how to run a covering argument \emph{only in the analysis}, since in this case the algorithm would use the CCE value of the original game $ (\Qover_{h}^k,\Qunder_{h}^k)  $ and this value cannot be controlled.
However, thanks to property (ii), it suffices to use the $ \epsilon $-cover \emph{in the algorithm}, since in this case the algorithm actually uses the CCE policy of the game $( \Qotilde, \Qutilde)$ from the finite $ \epsilon $-cover, and its value can be controlled by a union bound over the cover.  The small price we pay is that the resulting UCB/LCB are valid up to an $ 2 \epsilon $ error, which eventually goes into the regret bound. This error can be made negligible relative to the main terms in the regret by choosing a small enough~$ \epsilon $.

In summary,  the above algorithmic use of $ \epsilon $-cover appears crucial under our current framework. We leave as an intriguing open problem whether this algorithmic complication is in fact necessary or can be avoided by a more clever analysis. We also remark that the above issue does not exist in the \emph{tabular} setting, in which case the value functions $ (\Vover_{h}^k , \Vunder_{h}^k )$ are just a pair of finite-dimensional vectors and hence one can directly build an $ \epsilon $-cover for the relevant set of vectors.

\subsection{Theoretical Guarantees}

In each episode $ k $, Algorithm~\ref{alg:find_cce} computes a joint (correlated) policy $ \sigma_h^k $. As NE requires the policies to be in product form, we marginalize $ \sigma_h^k $ into a pair of independent policies $\pi_{h}^{k}(x):=\mP_{1}\sigma_{h}^{k}(x)$ and $\nu_{h}^{k}(x):=\mP_{2}\sigma_{h}^{k}(x)$ for each player. 
Our main theoretical result is the following bound on the total duality gap~\eqref{eq:gap} of these policy pairs. Recall that $ T=KH $ is the total number of timesteps.

\begin{thm}[Offline, Simultaneous Moves]
\label{thm:main_simu} Under Assumption \ref{assu:linear_bounded_simu},
there exists a constant $c>0$ such that the following holds for each
fixed $p\in(0,1)$. Set $\beta=cdH\sqrt{\iota}$ with $\iota:=\log(2dT/p)$
in Algorithm \ref{alg:simu}, and set $\epsilon=\frac{1}{KH}$ in
Algorithm~\ref{alg:find_cce}. Then with probability at least $1-p$,
Algorithm \ref{alg:simu} satisfies the bounds
\begin{align}
\label{eq:main_simu1}
V_{1}^{*,\nu^{k}}(x_{1}^{k})-V_{1}^{\pi^{k},*}(x_{1}^{k})
&\le \Vover_{1}^{k}(x_{1}^{k})-\Vunder_{1}^{k}(x_{1}^{k})+\frac{8}{K}, \quad \forall k\in[K]
\\
\label{eq:main_simu2}
\sum_{k=1}^{K} \left[\Vover_{1}^{k}(x_{1}^{k})-\Vunder_{1}^{k}(x_{1}^{k}) \right]
&\lesssim\sqrt{d^{3}H^{3}T\iota^{2}};
\end{align}
consequently, we have 
\begin{align}
\label{eq:main_simu3}
\mathrm{Gap}(K)
\lesssim\sqrt{d^{3}H^{3}T\iota^{2}}.
\end{align}
\end{thm}
The proof is given in Section~\ref{sec:proof_main_simu}. Below we provide discussion and remarks on this theorem.

\paragraph{Optimality of the bound:}

The theorem provides an (instance-independent) bound scaling with  $\sqrt{T}$. As the total duality gap reduces to the usual notion of regret in the special case MDPs, our bound is optimal in $ T $ in view of known minimax lower bounds for MDPs~\citep{lattimore2018bandit}. 
Also note that our bound is independent of the cardinality $|\mS|$ and $|\mA|$ of the state/action spaces, but rather depends only on dimension $ d $ of the feature space,  thanks to the use of function approximation. 
To investigate the tightness of the dependence of our bound on $d$ and $H$, we recall that our setting covers the standard tabular MDPs and linear bandits  as  special cases. A direct reduction from the known lower bounds on tabular MDPs gives a lower bound $\Omega(\sqrt{dH^2T})$ for the case of nonstationary transitions~\citep{jin2018q,azar2017minimax}. Our bound is off by a factor of $ \sqrt{H} $, which may be improved by using a ``Bernstein-type'' bonus term~\citep{azar2017minimax,jin2018q}. Results from linear bandits give the lower bound $\Omega(\sqrt{d^2T})$. The additional $ \sqrt{d} $ factor in our bound is due to a covering argument applied to the $ d $-dimensional feature space for establishing uniform concentration bounds.

\paragraph{Computational complexity:}

Our algorithm can be implemented efficiently, with computational and memory
complexities polynomial in $H,K,d$ and $\left|\mA\right|$. In particular, note that a CCE of a general-sum game can be found in polynomial time~\citep{papadimitriou2008computing,blum2008regret}.\footnote{This can be done by linear programming---as the inequalities in the definition~(\ref{eq:cce}) of CCE are linear in $\sigma$---or by self-playing a no-regret algorithm~\citep{blum2008regret}.}
Moreover, in Algorithm~\ref{alg:simu} we do \emph{not}  need to compute
$\Qover(x,\cdot,\cdot),\Vover(x)$ and $\sigmatilde(x)$ etc.\ for
\emph{all} $x \in \mS $; rather, we only need to do so for the states $\{ x_{h}^{k}\} $
actually encountered in the algorithm.
Similarly, we do not need to explicitly compute or store the $\epsilon$-net $\mQ_{\epsilon}$ in \findcce
~(Algorithm~\ref{alg:find_cce}). It suffices if we
can find an element in $\mQ_{\epsilon}$ that is $\epsilon$-close
to a given function in $\mQ$, which  can  be done efficiently
on the fly. Indeed, each function in $\mQ$ has a succinct representation
using $(w,A)\in\R^{d}\times\R^{d\times d}$. We can (implicitly)
maintain a covering of the space of $(w,A)$, and find a nearby element from
this covering when needed, which can be done in $ O(d^2) $ time via coordinate-wise rounding. See Appendix~\ref{sec:find_cce_implement} and Lemma~\ref{lem:fast_covering} therein for details.

\paragraph{Sample complexity and PAC guarantees:}

The regret bound in Theorem~\ref{thm:main_simu} can be converted
into a bound on the sample complexity. For simplicity we assume that
the initial state $x_{1}$ is fixed; for the general case where $x_{1}$ is sampled from a fixed distribution,
we can simply add an additional time step at the beginning of each
episode. After $K$ episodes, we choose, among the $ K $ policy pairs $ (\pi^k, \nu^k ), k\in[K] $ computed by Algorithm~\ref{alg:simu}, the pair $ (\pi^{k_0}, \nu^{k_0}) $ with the minimum gap between the UCB and LCB; that is,
\[
k_0 = \arg\min_{k\in [K]} \left\{ \Vover_{1}^{k}(x_{1})-\Vunder_{1}^{k}(x_{1}) \right\}.
\]
Note that the UCB/LCB, $\Vover_{1}^{k}(x_{1})$ and $ \Vunder_{1}^{k}(x_{1}) $, are computed by the algorithm and hence their values are known. This policy pair $ (\pi^{k_0}, \nu^{k_0}) $ satisfies the bound
\begin{align*}
&V_{1}^{*,\nu^{k_0}}(x_{1})-V_{1}^{\pi^{k_0},*}(x_{1}) \\
& \qquad \le   \Vover_{1}^{k_0}(x_{1})-\Vunder_{1}^{k_0}(x_{1})+\frac{8}{K} && \text{inequality~\eqref{eq:main_simu1}}\\
& \qquad \le   \frac{1}{K}\text{\ensuremath{\sum_{k=1}^{K}\left[\Vover_{1}^{k}(x_{1})-\Vunder_{1}^{k}(x_{1})\right]}} + \frac{8}{K}  && \text{min $ \le $ average}\\
& \qquad \lesssim  \sqrt{\frac{d^{3}H^{5}\iota^{2}}{T}}. && \text{inequality~\eqref{eq:main_simu2} divided by $ K=T/H $} 
\end{align*}
Therefore, we can find an $\epsilon$-approximate NE (meaning that
the last RHS is bounded by $\epsilon$) with a sample complexity of $T=O\left(\frac{d^{3}H^{5}\iota^{2}}{\epsilon^{2}}\right)$. By playing the policy pair $ (\pi^{k_0}, \nu^{k_0})  $ in all subsequent episodes, we obtain a PAC-type guarantee~\cite{kakade2003thesis} in the sense that an  $\epsilon$-approximate NE is played in all but $ O\left(\frac{d^{3}H^{5}\iota^{2}}{\epsilon^{2}}\right) $ timesteps.

\subsection{Turn-Based Games\label{sec:main_turn}}

In this section, we consider turn-based Markov games, which is a
special case of simultaneous-move Markov games. 
Algorithm~\ref{alg:simu} can be specialized to this setting. For completeness, we provide the resulting algorithm in Algorithm~\ref{alg:turn} in Appendix~\ref{sec:algos_turn_offline}. Note that for turn-based games,  the $ \findcce $  routine is simplified to the subroutines $ \findmax  $ and $ \findmin $ given in Algorithm~\ref{alg:find_max}, because each state is controlled by a single player and hence finding a CCE reduces to computing a maximizer or minimizer.

As a corollary of Theorem~\ref{thm:main_simu}, we have the following bound on the total duality gap, which is defined in the same way as in equation~\eqref{eq:gap}.
\begin{cor}[Offline, Turn-based]
\label{thm:main_turn} Under Assumption~\ref{assu:linear_bounded_turn},
there exists a constant $c>0$ such that, for each fixed $p\in(0,1)$,
by setting $\beta=cdH\sqrt{\iota}$ with $\iota:=\log(2dT/p)$ in
Algorithm \ref{alg:turn}, then with probability at least $1-p$,
Algorithm \ref{alg:turn} satisfies bound
\[
\mathrm{Gap}(K)\lesssim\sqrt{d^{3}H^{3}T\iota^{2}}.
\]
\end{cor}
 We prove this corollary in Appendix~\ref{sec:proof_main_turn}.

%% file: online.tex
%!TEX root =main.tex

\section{Main Results for the Online Setting\label{sec:main_oneline_simu}}

In this section, we consider the online setting, where we control
$\pone$ and play against an arbitrary (and potentially adversarial)
$\ptwo$. Our goal is to maximize the reward of $ \pone $. Below we describe the performance metrics, followed by our algorithms and theoretical guarantees.

\subsection{Setup and Performance Metrics}\label{sec:setup_online}

We consider the episodic setting as described in Section~\ref{sec:setup_offline}. Let  $\pi = (\pi^{k})$ and $\nu = (\nu^{k})$ be the policy sequences for $ \pone $ and $\ptwo$, respectively, where $ \nu $ is arbitrary. We do not know $ \ptwo $'s choice of $ \nu $ nor the Markov model of the game a priori, and would like learn a good policy $ \pi $ online so as to optimize the reward $ \sum_k V_{1}^{\pi^k,\nu^k} $ received by $ \pone $ over $ K $ episodes. To this end, we are interested in bounding, for each $ \nu $, the total (expected)  regret
\begin{equation} \label{eq:regret}
\textrm{Regret}_{\nu}(K):=\sum_{k=1}^{K}\big[V_{1}^{*}(x_{1}^{k})-V_{1}^{\pi^{k},\nu^{k}}(x_{1}^{k})\big],
\end{equation}
where $x_{1}^{k}$ is the (arbitrary) initial state in the $k$-th episode. If we can obtain a bound on $ \textrm{Regret}_{\nu}(K) $ that scales sublinearly with $ K $ for all $ \nu $, then we are guaranteed that regardless of $ \nu $, the reward collected by $ \pone $ is no worse (in the long run) than its optimal worst-case reward, that is, the NE value $ V_{1}^{*} $.

We note that a special case of the above  setting is when  $\ptwo$ is omniscient and always
plays the best response to $ \pone $'s policy, i.e., 
\[
\nu^{k}= \br(\pi^k) \in \arg\min_{\nu'\in\simplex}V_{1}^{\pi^{k},\nu'}(x_1^k), \quad \forall k\in[K].
\]
Note that in this case, we have  $V_{1}^{\pi^{k},\nu^{k}}(x_{1}^{k})=V_{1}^{\pi^{k},*}(x_{1}^{k})$
by definition.

\subsection{Algorithm}\label{sec:alg_online}

We adapt the \omvi\ algorithm to the online setting, as given in Algorithm~\ref{alg:online_simu}. This algorithm can be viewed as a one-sided version of Algorithm~\ref{alg:simu}: we compute least-squares estimate for the linear coefficients and then construct UCBs for the value functions---we do not need to construct LCBs as $ \ptwo $ is not controlled by us.  Constructing the UCBs is done by finding the NE of the \emph{zero-sum} matrix game with the payoff matrix $Q_{h}^{k}(x,\cdot,\cdot)$.
%Recalling the definition of NE, we see that the pair $(\pi_{h}^{k}(x),B_{0}) \in\simplex\times\simplex$ computed in the algorithm
%satisfies
%\begin{align}
%\E_{a\sim\pi_{h}^{k}(x),b\sim B_{0}}\left[Q_{h}^{k}(x,a,b)\right] & =\max_{A\in\simplex}\E_{a\sim A,b\sim\nu_{h}^{k}(x)}\left[Q_{h}^{k}(x,a,b)\right]=\min_{B\in\simplex}\E_{a\sim\pi_{h}^{k}(x),b\sim B}\left[Q_{h}^{k}(x,a,b)\right], %\label{eq:nash_online}
%\end{align}
%for each $ x \in \mS$. 

\begin{algorithm}[ht]
\caption{\omvi\ (Simultaneous Move, Online)\label{alg:online_simu}}

\begin{algorithmic}[1]
	
\State \textbf{Input:} bonus parameter $ \beta>0 $.
	
\For{ episode $k=1,2,\ldots,K$ } 

\State Receive initial state $x_{1}^{k}$.

\For{step $h=H,H-1,\ldots,2,1$ } \Comment{update policy}

\State $\Lambda_{h}^{k}\leftarrow\sum_{\tau=1}^{k-1}\phi(x_{h}^{\tau},a_{h}^{\tau},b_{h}^{\tau})\phi(x_{h}^{\tau},a_{h}^{\tau},b_{h}^{\tau})^{\top}+I.$

\State $w_{h}^{k}\leftarrow(\Lambda_{h}^{k})^{-1}\sum_{\tau=1}^{k-1}\phi(x_{h}^{\tau},a_{h}^{\tau},b_{h}^{\tau})\left[r_{h}(x_{h}^{\tau},a_{h}^{\tau},b_{h}^{\tau})+V_{h+1}^{k}(x_{h+1}^{\tau})\right]$.

\State $Q_{h}^{k}(\cdot,\cdot,\cdot)\leftarrow\projH\left\{ (w_{h}^{k})^{\top}\phi(\cdot,\cdot,\cdot)+\beta\sqrt{\phi(\cdot,\cdot,\cdot)^{\top}(\Lambda_{h}^{k})^{-1}\phi(\cdot,\cdot,\cdot)}\right\} $.

\State For each $x$, let $(\pi_{h}^{k}(x),B_{0})$ be the
NE of the matrix game with payoff matrix $Q_{h}^{k}(x,\cdot,\cdot)$.

\State $V_{h}^{k}(\cdot)\leftarrow\E_{a\sim\pi_{h}^{k}(\cdot),b\sim B_{0}}\left[Q_{h}^{k}(\cdot,a,b)\right].$

\EndFor

\For{step $h=1,2,\ldots,H$ } \Comment{execute policy}

\State $\pone$ take action $a_{h}^{k}\sim\pi_{h}^{k}(x_{h}^{k})$.

\State Let $\ptwo$ play; denote its action by $b_{h}^{k}$.

\State Observe next state $x_{h+1}^{k}$.
\EndFor
\EndFor

%\State \textbf{Output:} policy sequence $ \pi^k, k\in[K]  $
\end{algorithmic}
\end{algorithm}

Due to the one-sided nature of the online setting, some of the difficulties in the offline setting---pertaining to general-sum games and CCE---no longer exist here. In particular, Algorithm~\ref{alg:online_simu} no longer requires the $ \findcce $ subroutine that makes use of an $\epsilon$-cover. Technically, this is due to the fact that zero-sum matrix games are  more well-behaved than general-sum games. In particular, the value of a zero-sum game \emph{is} Lipschitz in the payoff matrix, hence  uniform concentration can be established in a more straightforward manner (cf.\ the discussion in Section~\ref{sec:alg_offline}).

\subsection{Regret Bound Guarantees}

We establish the following bound on the total regret~\eqref{eq:regret} achieved by Algorithm~\ref{alg:online_simu}. 
\begin{thm}[Online, Simultaneous Move]
\label{thm:main_online_simu} Under Assumption~\ref{assu:linear_bounded_simu},
there exists a constant $c>0$ such that the following holds for each
fixed $p\in(0,1)$ and any policy sequence $ \nu $ for $ \ptwo $.
Set $\beta=cdH\sqrt{\iota}$ with $\iota:=\log(2dT/p)$.
Then with probability at least $1-p$, Algorithm~\ref{alg:online_simu} achieves the regret bound 
\[
\mathrm{Regret}_{\nu}(K)\lesssim\sqrt{d^{3}H^{3}T\iota^{2}}.
\]
\end{thm}
The proof is given in Appendix~\ref{sec:proof_main_online_simu}. Note that the regret bound holds for any policy $ \nu $ of $ \ptwo  $ and any initial states $ \{x_1^k\} $. Moreover, the bound is sublinear in $ T $---scaling with $ \sqrt{T} $ in particular---and depends polynomially on $ d $ and $ H $. 
As our regret reduces to the standard regret notion in the special cases of MDPs and linear bandits, the discussion in Section~\ref{sec:alg_offline} on the optimality of bounds, also applies here. 

We remark that the above bound provides a uniform guarantee for $ \pone $'s performance, regardless of the policy of the opponent $\ptwo$. An interesting future direction is to achieve a more refined guarantee that \emph{exploits a weak opponent}. In particular, such a guarantee would involve a stronger notion of regret in which, instead of competing with the Nash value $ \sum_{k} V_{1}^{*}(x_{1}^{k}) $ as in the current definition~\eqref{eq:regret},  one competes against the value $  \max_{\pi} \sum_{k=1}^K V_1^{\pi, \nu^k}(x_1^k) $ achieved by the best fixed policy in hindsight. We believe doing so would require modifying the algorithm, which is left to future work.

\subsection{Turn-Based Games\label{sec:main_online_turn}}

The algorithm above can be specialized to online turn-based games. 
For completeness we provide resulting algorithm in Appendix~\ref{sec:algos_turn_offline} as Algorithm~\ref{alg:online_turn}.
Note that in the turn-based setting, we only need to solve a unilateral
maximization or minimization problem, rather than solving zero-sum
games as is needed in the simultaneous-move setting. 

As an immediate corollary of Theorem~\ref{thm:main_online_simu}, we have the following regret bound for turn-based games in the online setting. 
\begin{cor}[Online, Turn-based]
\label{thm:main_online_turn}
Under Assumption~\ref{assu:linear_bounded_turn},
there exists a constant $c>0$ such that the following holds for each fixed $p\in(0,1)$ and any policy sequence $ \nu $ for $ \ptwo $.
Set $\beta=cdH\sqrt{\iota}$ with $\iota:=\log(2dT/p)$ in
Algorithm~\ref{alg:online_turn}. Then with probability at least $1-p$, Algorithm
\ref{alg:online_turn} achieves the regret bound 
\[
\mathrm{Regret}_{\nu}(K)\lesssim\sqrt{d^{3}H^{3}T\iota^{2}}.
\]
\end{cor}
We prove this corollary in Appendix~\ref{sec:proof_main_online_turn}.

%% file: proofthm.tex
%!TEX root =main.tex

\section{Proof of Theorem \ref{thm:main_simu}\label{sec:proof_main_simu}}

In this section, we prove Theorem~\ref{thm:main_simu} for the offline
setting of simultaneous games. We shall make use of the technical
lemmas given in Appendix~\ref{sec:tech_lemma}. For clarity of exposition, we denote by
$\phi_{h}^{k}:=\phi(x_{h}^{k},a_{h}^{k},b_{h}^{k})$ the feature vector encountered in the $ h $-th step of the $ k $-th episode. Our proof consists of five steps:
\begin{enumerate}[label=\roman*]
	\item \textbf{Uniform concentration:} We begin by  showing that an empirical estimate of the transition kernel $ \Pr_h $, when acting on the value functions maintained by the algorithm, concentrates around its expectation. See Section~\ref{sec:proof_main_concentration}.
	\item \textbf{Least-squares estimation error:} Using the above concentration result, we derive high probability bounds on the errors of our least-squares estimates of the true Q functions $ Q_h^{\pi,\nu} $, recursively in the timestep $h $. See Section~\ref{sec:proof_main_estimation}.
	\item \textbf{UCB and LCB:} We next show that the UCBs and LCBs constructed in the algorithms are indeed valid bounds on the true value functions $ V_h^{\pi,*} $ and $ V_h^{*,\nu} $.  See Section~\ref{sec:proof_main_UCB}.
	\item \textbf{Recursive decomposition of duality gap:} We derive a recursive formula for the difference between the UCB and LCB in terms of the timestep $ h $. This difference in turn bounds the duality gap of interest. See Section~\ref{sec:proof_main_recursive}.
	\item \textbf{Establishing final bound:} Bounding each term in the above recursive decomposition in terms of the least-squares estimation errors,  we establish the desired bound on the total duality gap, thereby completing the proof of the theorem. See Section~\ref{sec:proof_main_together}.
\end{enumerate}
Below we provide the details of each step.

\subsection{Uniform Concentration}\label{sec:proof_main_concentration}

The quantity $ \sum_{\tau\in[k-1]}\phi_{h}^{\tau}\Vover_{h+1}^{k}(x_{h+1}^{\tau}) $ can be viewed as an empirical estimate of the unknown population quantity $ \sum_{\tau\in[k-1]}\phi_{h}^{\tau}\big(\Pr_{h}\Vover_{h+1}^{k}\big)(x_{h}^{\tau},a_{h}^{\tau},b_{h}^{\tau})$. To control the least-squares estimation error, we need to show that the empirical estimate concentrates around its population counterpart. The main challenge in doing so is that $\Vover_{h+1}^{k}$ is constructed using data from previous episodes and hence depends on  $ \phi_h^\tau$ for all $ \tau \in [k-1] $. We overcome this issue by noting that $ \Vover_{h+1}^{k} $ is computed using the CCE of a finite class of games with payoff matrices in the $ \epsilon $-net $\mQ_{\epsilon}\times\mQ_{\epsilon}$, as is done in $ \findcce $. Therefore, we can prove a concentration bound valid uniformly over this class of games and thereby establish following  concentration result. Here we recall that $ \| v\|_A := \sqrt{v^\top A v}$ denotes the weighted $ \ell_2 $ norm of a vector $ v $.
\begin{lem}[Concentration]
\label{lem:concentration_simu} Under the setting of Theorem~\ref{thm:main_simu},
for each $p\in(0,1)$, the following event $\mathfrak{E}$ holds with
probability at least $1-p/2$:
\begin{align*}
\left\Vert \sum_{\tau\in[k-1]}\phi_{h}^{\tau}\left[\Vover_{h+1}^{k}(x_{h+1}^{\tau})-\left(\Pr_{h}\Vover_{h+1}^{k}\right)(x_{h}^{\tau},a_{h}^{\tau},b_{h}^{\tau})\right]\right\Vert _{(\Lambda_{h}^{k})^{-1}} & \lesssim dH\sqrt{\log(dT/p)},\qquad\forall (k,h)\in[K]\times[H],\\
\left\Vert \sum_{\tau\in[k-1]}\phi_{h}^{\tau}\left[\Vunder_{h+1}^{k}(x_{h+1}^{\tau})-\left(\Pr_{h}\Vunder_{h+1}^{k}\right)(x_{h}^{\tau},a_{h}^{\tau},b_{h}^{\tau})\right]\right\Vert _{(\Lambda_{h}^{k})^{-1}} & \lesssim dH\sqrt{\log(dT/p)},\qquad\forall (k,h)\in[K]\times[H].
\end{align*}
\end{lem}
\begin{proof}
Fix $ (k,h) \in [K] \times [H]$. Let 
\begin{align}\label{eq:filtration}
\mathcal{F}_{\tau-1}:=\sigma(x_{\cdot}^{1},a_{\cdot}^{1},b_{\cdot}^{1}, \ldots, x_{\cdot}^{\tau-1},a_{\cdot}^{\tau-1}, b_{\cdot}^{\tau-1},  x_1^\tau, a_1^\tau, b_1^\tau, \ldots, x_h^\tau, a_h^\tau, b_h^\tau )
\end{align}
be the $ \sigma $-algebra generated by the data from the first $ \tau-1 $ episodes \emph{plus} that from the first $h$ steps of the $ \tau $-th episode. We note that as actions are randomized, they must also be  included  in the definition of the above filtration, unlike in the MDP setting.
Also note that $ \phi_h^\tau, x_h^\tau, a_h^\tau, b_h^\tau \in \mathcal{F}_{\tau-1}$ and $ x_{h+1}^\tau \in \mathcal{F}_\tau. $
%Also that $\Qover_h^\tau, \Qunder_h^\tau, \Vover_h^\tau, \Vunder_h^\tau$ and $ \sigma_h^\tau $ are in $\mathcal{F}_{\tau-1}$, as they are computed using data from the previous $ \tau-1 $ episodes.

Fix a pair $\left(\Qotilde,\Qutilde\right)$ in the $\epsilon$-net
$\mQ_{\epsilon}\times\mQ_{\epsilon}$. For each $x\in\mS$, let $\sigmatilde(x)$
be the CCE of $\left(\Qotilde(x,\cdot,\cdot),\Qutilde(x,\cdot,\cdot)\right)$
in the sense of equation~\eqref{eq:cce}, and set $\Votilde(x):=\E_{(a,b)\sim\sigmatilde(x)}\left[\Qotilde(x,a,b)\right]$.
The random variable $\Votilde(x_{h+1}^{\tau})-(\Pr_{h}\Votilde)(x_{h}^{\tau},a_h^\tau,b_h^\tau)$, when conditioned on $ \mathcal{F}_{\tau-1} $,
is zero-mean and $H$-bounded. Applying Lemma~\ref{lem:self_normalized}
gives
\[
\left\Vert \sum_{\tau\in[k-1]}\phi_{h}^{\tau}\left[\Votilde(x_{h+1}^{\tau})-\left(\Pr_{h}\Votilde\right)(x_{h}^{\tau},a_{h}^{\tau},b_{h}^{\tau})\right]\right\Vert _{(\Lambda_{h}^{k})^{-1}}\lesssim dH\sqrt{\log(dT/p)}
\]
with probability at least $2^{-\Omega(d^{2}\log(dT/p))}$. Now note
that $\left|\mQ_{\epsilon}\times\mQ_{\epsilon}\right|=(\mathcal{N}_{\epsilon})^{2}\le4\left(1+\frac{8H\sqrt{dk}}{\epsilon}\right)^{2d}\left(1+\frac{\beta^{2}\sqrt{d}}{\epsilon^{2}}\right)^{2d^{2}}$
by Lemma~\ref{lem:covering_Q}. By a union bound and the choice that $ \epsilon = 1/(kH)$, the above inequality
holds for all $\left(\Qotilde,\Qutilde\right)\in\mQ_{\epsilon}\times\mQ_{\epsilon}$
with probability at least $1-p/2$.

Now, for the pair $\left(\Qover_{h+1}^{k},\Qunder_{h+1}^{k}\right)$, which is in $\mQ\times\mQ$
 by Lemma~\ref{lem:alg_bounded}, let $\left(\Qotilde,\Qutilde\right)\in\mQ_{\epsilon}\times\mQ_{\epsilon}$
be the pair in the net as chosen in $\findcce$. Recall that by construction we have  $\left\Vert \Qotilde-\Qover_{h}^{k}\right\Vert _{\infty}\le\epsilon$, 
and $\left\Vert \Qutilde-\Qunder_{h}^{k}\right\Vert _{\infty}\le\epsilon$ and $\Vover_{h+1}^{k}(x)=\E_{(a,b)\sim\sigmatilde(x)}\left[\Qover_{h+1}^{k}(x,a,b)\right]$.
Therefore, the difference $\Delta(x):=\Vover_{h+1}^{k}(x)-\Votilde(x)$
satisfies
\begin{align*}
\left|\Delta(x)\right| & =\left|\E_{(a,b)\sim\sigmatilde(x)}\left[\Qover_{h+1}^{k}(x,a,b)-\Qotilde(x,a,b)\right]\right|\\
 & \le\E_{(a,b)\sim\sigmatilde(x)}\left|\Qover_{h+1}^{k}(x,a,b)-\Qotilde(x,a,b)\right|
 \le\epsilon,\qquad\forall x\in\mS.
\end{align*}
It follows that 
\begin{align*}
 & \left\Vert \sum_{\tau\in[k-1]}\phi_{h}^{\tau}\left[\Vover_{h+1}^{k}(x_{h+1}^{\tau})-\left(\Pr_{h}\Vover_{h+1}^{k}\right)(x_{h}^{\tau},a_{h}^{\tau},b_{h}^{\tau})\right]\right\Vert _{(\Lambda_{h}^{k})^{-1}}\\
 & \quad \le  \left\Vert \sum_{\tau\in[k-1]}\phi_{h}^{\tau}\left[\Votilde(x_{h+1}^{\tau})-\left(\Pr_{h}\Votilde\right)(x_{h}^{\tau},a_{h}^{\tau},b_{h}^{\tau})\right]\right\Vert _{(\Lambda_{h}^{k})^{-1}}+\left\Vert \sum_{\tau\in[k-1]}\phi_{h}^{\tau}\left[\Delta(x_{h+1}^{\tau})-\left(\Pr_{h}\Delta\right)(x_{h}^{\tau},a_{h}^{\tau},b_{h}^{\tau})\right]\right\Vert _{(\Lambda_{h}^{k})^{-1}}\\
  & \quad \lesssim
  dH\sqrt{\log(dT/p)}+\epsilon\sum_{\tau\in[k-1]}\left\Vert \phi_{h}^{\tau}\right\Vert _{(\Lambda_{h}^{k})^{-1}}\\
  & \quad  \le dH\sqrt{\log(dT/p)}+\epsilon k,
\end{align*}
where the last step follows from $\Lambda_{h}^{k}\succeq I$ and $\left\Vert \phi_{h}^{\tau}\right\Vert \le1$.
Recalling our choice $\epsilon=\frac{1}{KH}$ proves the first inequality
in the lemma. The second inequality can be proved in a similar fashion.
\end{proof}

\subsection{Least-squares Estimation Error}\label{sec:proof_main_estimation}

Here we bound the difference between the algorithm's action-value
functions (without bonus) and the true action-value functions of any
policy pair $(\pi,\nu)$, recursively in terms of the step $h$.
\begin{lem}[Least-squares Error Bound]
\label{lem:ls_error_simu}The quantities $\{\wover_{h}^{k},\text{\ensuremath{\wunder}}_{h}^{k},\Vover_{h}^{k},\Vunder_{h}^{k}\}$
in Algorithm~\ref{alg:simu} satisfy the following. If $\beta=dH\sqrt{\iota},$ where $\iota=\log(2dT/p),$ 
then on the event $\mathfrak{E}$ in Lemma \ref{lem:concentration_simu},
we have for all $(x,a,b,h,k)\in\mS\times\mA\times\mA\times[H]\times[K]$
and any policy pair $(\pi,\nu)$:
\begin{subequations}
\label{eq:ls_error_simu}
\begin{align}
\left|\left\langle \phi(x,a,b),\wover_{h}^{k}\right\rangle -Q_{h}^{\pi,\nu}(x,a,b)-\Pr_{h}(\Vover_{h+1}^{k}-V_{h+1}^{\pi,\nu})(x,a,b)\right| & \le\rho_{h}^{k}(x,a,b),\label{eq:wover}\\
\left|\left\langle \phi(x,a,b),\wunder_{h}^{k}\right\rangle -Q_{h}^{\pi,\nu}(x,a,b)-\Pr_{h}(\Vunder_{h+1}^{k}-V_{h+1}^{\pi,\nu})(x,a,b)\right| & \le\rho_{h}^{k}(x,a,b),\label{eq:wunder}
\end{align}
\end{subequations}
where $\rho_{h}^{k}(x,a,b):=\beta\left\Vert \phi(x,a,b)\right\Vert _{(\Lambda_{h}^{k})^{-1}}.$
\end{lem}
\begin{proof}
We only prove the first inequality~(\ref{eq:wover}). The second inequality
can be proved in a similar fashion.

By Lemma~\ref{lem:linearity_simu} and Bellman equation we have the equality 
\[
(\phi_{h}^{\tau})^{\top}w_{h}^{\pi,\nu}=Q_{h}^{\pi,\nu}(x_{h}^{\tau},a_{h}^{\tau},b_{h}^{\tau})=r_{h}(x_{h}^{\tau},a_{h}^{\tau},b_{h}^{\tau})+(\Pr_{h}V_{h+1}^{\pi,\nu})(x_{h}^{\tau},a_{h}^{\tau},b_{h}^{\tau})
\]
 for all $\tau\in[k-1]$. Multiplying the above equality by $\left(\Lambda_{h}^{k}\right)^{-1}\phi_{h}^{\tau}$
and summing over $\tau$, we obtain that 
\begin{align*}
w_{h}^{\pi,\nu}-\left(\Lambda_{h}^{k}\right)^{-1}w_{h}^{\pi,\nu} & =\left(\Lambda_{h}^{k}\right)^{-1}\left(\sum_{\tau\in[k-1]}\phi_{h}^{\tau}(\phi_{h}^{\tau})^{\top}\right)w_{h}^{\pi,\nu}\\
 & =\left(\Lambda_{h}^{k}\right)^{-1}\sum_{\tau\in[k-1]}\phi_{h}^{\tau}\cdot\left[r_{h}(x_{h}^{\tau},a_{h}^{\tau},b_{h}^{\tau})+(\Pr_{h}V_{h+1}^{\pi,\nu})(x_{h}^{\tau},a_{h}^{\tau},b_{h}^{\tau})\right],
\end{align*}
where the first equality above holds because $\sum_{\tau\in[k-1]}\phi_{h}^{\tau}(\phi_{h}^{\tau})^{\top}=\Lambda_{h}^{k}-I$.
On the other hand, recall that by algorithm specification we have
$\wover_{h}^{k}=(\Lambda_{h}^{k})^{-1}\sum_{\tau\in[k-1]}\phi_{h}^{\tau}\cdot\left[r_{h}(x_{h}^{\tau},a_{h}^{\tau},b_{h}^{\tau})+\Vover_{h+1}^{k}(x_{h+1}^{\tau})\right].$
It follows that
\begin{align*}
\wover_{h}^{k}-w_{h}^{\pi,\nu} & =-\left(\Lambda_{h}^{k}\right)^{-1}w_{h}^{\pi,\nu}+(\Lambda_{h}^{k})^{-1}\sum_{\tau\in[k-1]}\phi_{h}^{\tau}\cdot\left[\Vover_{h+1}^{k}(x_{h+1}^{\tau})-(\Pr_{h}V_{h+1}^{\pi,\nu})(x_{h}^{\tau},a_{h}^{\tau},b_{h}^{\tau})\right]\\
 & =-\underbrace{\left(\Lambda_{h}^{k}\right)^{-1}w_{h}^{\pi,\nu}}_{q_{1}}+\underbrace{(\Lambda_{h}^{k})^{-1}\sum_{\tau\in[k-1]}\phi_{h}^{\tau}\cdot\left[\Vover_{h+1}^{k}(x_{h+1}^{\tau})-(\Pr_{h}\Vover_{h+1}^{k})(x_{h}^{\tau},a_{h}^{\tau},b_{h}^{\tau})\right]}_{q_{2}}\\
 & \qquad\qquad+\underbrace{(\Lambda_{h}^{k})^{-1}\sum_{\tau\in[k-1]}\phi_{h}^{\tau}\cdot\left[\Pr_{h}(\Vover_{h+1}^{k}-V_{h+1}^{\pi,\nu})(x_{h}^{\tau},a_{h}^{\tau},b_{h}^{\tau})\right]}_{q_{3}}.
\end{align*}
whence for each $(x,a,b)$:
\[
\left\langle \phi(x,a,b),\wover_{h}^{k}\right\rangle -Q_{h}^{\pi,\nu}(x,a,b)=\left\langle \phi(x,a,b),q_{1}+q_{2}+q_{3}\right\rangle .
\]
We apply Cauchy-Schwarz to bound each RHS term:
\begin{enumerate}
\item First term: we have
\begin{align*}
\left|\left\langle \phi(x,a,b),q_{1}\right\rangle \right| & \le\left\Vert w_{h}^{\pi,\nu}\right\Vert _{(\Lambda_{h}^{k})^{-1}}\cdot\left\Vert \phi(x,a,b)\right\Vert _{(\Lambda_{h}^{k})^{-1}}\\
 & \le\left\Vert w_{h}^{\pi,\nu}\right\Vert \cdot\left\Vert \phi(x,a,b)\right\Vert _{(\Lambda_{h}^{k})^{-1}}\lesssim H\sqrt{d}\cdot\left\Vert \phi(x,a,b)\right\Vert _{(\Lambda_{h}^{k})^{-1}},
\end{align*}
where the last two steps follow from $\Lambda_{h}^{k}\succeq I$ and
$\left\Vert w_{h}^{\pi,\nu}\right\Vert \lesssim H\sqrt{d}$ (Lemma~\ref{lem:true_bounded}).
\item Second term: by Lemma~\ref{lem:concentration_simu} we have 
\[
\left|\left\langle \phi(x,a,b),q_{2}\right\rangle \right|\lesssim dH\sqrt{\log(dT/p)}\cdot\left\Vert \phi(x,a,b)\right\Vert _{(\Lambda_{h}^{k})^{-1}}.
\]

\item Third term: recalling that $\sum_{\tau\in[k-1]}\phi_{h}^{\tau}\left(\phi_{h}^\tau\right)^{\top}=\Lambda_{h}^{k}-I$
and $\Pr_{h}(\cdot|x_{h}^{\tau},a_{h}^{\tau},b_{h}^{\tau})=\left(\phi_{h}^{\tau}\right)^{\top}\mu_{h}(\cdot)$,
we have 
\begin{align*}
&\left\langle \phi(x,a,b),q_{3}\right\rangle  \\
& =\left\langle \phi(x,a,b),(\Lambda_{h}^{k})^{-1}\sum_{\tau\in[k-1]}\phi_{h}^{\tau}\left(\phi_{h}^{\tau}\right)^{\top}\int(\Vover_{h+1}^{k}-V_{h+1}^{\pi,\nu})(x')\dup\mu_{h}(x')\right\rangle \\
 & =\left\langle \phi(x,a,b),\int(\Vover_{h+1}^{k}-V_{h+1}^{\pi,\nu})(x')\dup\mu_{h}(x')\right\rangle -\left\langle \phi(x,a,b),(\Lambda_{h}^{k})^{-1}\int(\Vover_{h+1}^{k}-V_{h+1}^{\pi,\nu})(x')\dup\mu_{h}(x')\right\rangle \\
 & =\Pr_{h}(\Vover_{h+1}^{k}-V_{h+1}^{\pi,\nu})(x,a,b)+\underbrace{\left\langle \phi(x,a,b),(\Lambda_{h}^{k})^{-1}\int(\Vover_{h+1}^{k}-V_{h+1}^{\pi,\nu})(x')\dup\mu_{h}(x')\right\rangle }_{p_{2}}.
\end{align*}
Note that in the above equality we make crucial use of the linearity assumption on the transition kernel.
The term $ p_2 $ above satisfies the bound
\[
\left|p_{2}\right|\lesssim\left\Vert \phi(x,a,b)\right\Vert _{(\Lambda_{h}^{k})^{-1}}\cdot H\sqrt{d},
\]
where we use the facts that $\Lambda_{h}^{k}\succeq I,$ $\left\Vert \mu_{h}(\mS)\right\Vert \le\sqrt{d}$,
$\left|\Vover_{h+1}^{k}(\cdot)\right|\le H,$ and $\left|V_{h+1}^{\pi,\nu}(\cdot)\right|\le H$$.$
\end{enumerate}
Combining, we obtain 
\[
\left|\left\langle \phi(x,a,b),\wover_{h}^{k}\right\rangle -Q_{h}^{\pi,\nu}(x,a,b)-\Pr_{h}(\Vover_{h+1}^{k}-V_{h+1}^{\pi,\nu})(x,a,b)\right|\lesssim dH\left\Vert \phi(x,a,b)\right\Vert _{(\Lambda_{h}^{k})^{-1}}\le\beta\left\Vert \phi(x,a,b)\right\Vert _{(\Lambda_{h}^{k})^{-1}}
\]
under our choice of $\beta\asymp dH\sqrt{\iota}$. This completes
the proof of the inequality~(\ref{eq:wover}) in the lemma.
\end{proof}

The above lemma can be specialized to the value functions of the best
response (cf.~Remark~\ref{rem:linear_best_response}); for example, it holds that 
\[
\left|\left\langle \phi(x,a,b),\wover_{h}^{k}\right\rangle -Q_{h}^{\pi,*}(x,a,b)-\Pr_{h}(\Vover_{h+1}^{k}-V_{h+1}^{\pi,*})(x,a,b)\right|\le\rho_{h}^{k}(x,a,b).
\]
We will make use of this bound and its variants in the subsequent proof.

\subsection{Upper and Lower Confidence Bounds}\label{sec:proof_main_UCB}

With the above bounds on the estimation errors, we can show that $ \Vunder_h^{k} $ and $ \Vover_h^{k} $ constructed in the algorithm are indeed lower and upper bounds for the true value function. To this end, we state a simple lemma first.
\begin{lem}[Algorithm~\ref{alg:find_cce} Finds $2\epsilon$-CCE]
\label{lem:eps_cce}For each $(k,h,x)$, $\sigma_{h}^{k}(x)$ is
an $2\epsilon$-CCE of $\left(\Qover_{h}^{k}(x,\cdot,\cdot),\Qunder_{h}^{k}(x,\cdot,\cdot)\right)$
in the sense that 
\begin{align*}
\E_{(a,b)\sim\sigmatilde(x)}\left[\Qover_{h}^{k}(x,a,b)\right] & \ge\E_{b\sim\mP_{2}\sigmatilde(x)}\left[\Qover_{h}^{k}(x,a',b)\right]-2\epsilon,\qquad\forall a'\in\mA,\\
\E_{(a,b)\sim\sigmatilde(x)}\left[\Qunder_{h}^{k}(x,a,b)\right] & \le\E_{a\sim\mP_{1}\sigmatilde(x)}\left[\Qover_{h}^{k}(x,a,b')\right]+2\epsilon,\qquad\forall b'\in\mA.
\end{align*}
\end{lem}
\begin{proof}
Let $\left(\Qotilde,\Qutilde\right)$ be the elements in the $\epsilon$-net
that are closest to $\left(\Qover_{h}^{k},\Qunder_{h}^{k}\right)$,
as specified in Algorithm~\ref{alg:find_cce}. This means that $\left|\Qover_{h}^{k}(x,a,b)-\Qotilde(x,a,b)\right|\le\epsilon$
and $\left|\Qunder_{h}^{k}(x,a,b)-\Qutilde(x,a,b)\right|\le\epsilon$
for all $(x,a,b)$. Fix an arbitrary $x\in\mS$. Because $\sigma_{h}^{k}(x)=\sigmatilde(x)$
is an CCE of $\left(\Qotilde(x,\cdot,\cdot),\Qutilde(x,\cdot,\cdot)\right)$,
we have for all $a'\in\mA$:
\begin{align*}
\E_{(a,b)\sim\sigmatilde(x)}\left[\Qover_{h}^{k}(x,a,b)\right] & =\E_{(a,b)\sim\sigmatilde(x)}\left[\Qotilde_{h}^{k}(x,a,b)\right]+\E_{(a,b)\sim\sigmatilde(x)}\left[\Qover_{h}^{k}(x,a,b)-\Qotilde_{h}^{k}(x,a,b)\right]\\
 & \ge\E_{b\sim\mP_{2}\sigmatilde(x)}\left[\Qotilde_{h}^{k}(x,a',b)\right]-\epsilon\\
 & =\E_{b\sim\mP_{2}\sigmatilde(x)}\left[\Qover_{h}^{k}(x,a',b)\right]+\E_{b\sim\mP_{2}\sigmatilde(x)}\left[\Qotilde_{h}^{k}(x,a',b)-\Qover_{h}^{k}(x,a',b)\right]-\epsilon\\
 & \ge\E_{b\sim\mP_{2}\sigmatilde(x)}\left[\Qover_{h}^{k}(x,a',b)\right]-2\epsilon.
\end{align*}
This proves the first inequality in the lemma. The second inequality
can be proved in a similar fashion.
\end{proof}

We can now establish the UCB and LCB properties.
\begin{lem}[UCB and LCB]
\label{lem:ucb_simu} Under the setting of Theorem~\ref{thm:main_simu},
on the event $\mathfrak{E}$ in Lemma~\ref{lem:concentration_simu},
we have for each $(x,a,b,k,h)$:
\begin{align*}
\Qunder_{h}^{k}(x,a,b)-2(H-h+1)\epsilon\overset{\text{(a)}}{\le} & Q_{h}^{\pi^{k},*}(x,a,b)\overset{\text{(b)}}{\le}Q_{h}^{*,\nu^{k}}(x,a,b)\overset{\text{(c)}}{\le}\Qover_{h}^{k}(x,a,b)+2(H-h+1)\epsilon
\end{align*}
and
\begin{align*}
\Vunder_{h}^{k}(x)-2(H-h+2)\epsilon\overset{\text{(i)}}{\le} & V_{h}^{\pi^{k},*}(x)\overset{\text{(ii)}}{\le}V_{h}^{*,\nu^{k}}(x)\overset{\text{(iii)}}{\le}\Vover_{h}^{k}(x)+2(H-h+2)\epsilon.
\end{align*}
\end{lem}
\begin{proof}
The inequalities (b) and (ii) follow from Proposition \ref{prop:weak_duality_simu}.
Below we only prove the upper bounds (c) and (iii). The lower bounds
(a) and (i) can be proved in a similar fashion.

We fix $k$ and perform induction on $h$. The base case $h=H+1$
holds since the terminal cost is zero. Now assume that the bounds
(c) and (iii) hold for step $h+1$; that is, $\Qover_{h+1}^{k}(x,a,b)\ge Q_{h+1}^{*,\nu^{k}}(x,a,b)-2(H-h)\epsilon$
and $\Vover_{h+1}^{k}(x)\ge V_{h+1}^{*,\nu^{k}}(x)-2(H-h+1)\epsilon$
for all $(x,a,b)$. By inequality~(\ref{eq:wover}) in Lemma~\ref{lem:ls_error_simu}
applied to $(\widetilde{\pi},\nu^{k})$ with $\widetilde{\pi}$ being
the best response to $\nu^{k}$, we have for each $(x,a,b)$:
\begin{align*}
 & \left|\left\langle \phi(x,a,b),\wover_{h}^{k}\right\rangle -Q_{h}^{*,\nu^{k}}(x,a,b)-\Pr_{h}\left(\Vover_{h+1}^{k}-V_{h+1}^{*,\nu^{k}}\right)(x,a,b)\right|\le\rho_{h}^{k}(x,a,b),
\end{align*}
whence
\begin{align*}
 & \left\langle \phi(x,a,b),\wover_{h}^{k}\right\rangle +\rho_{h}^{k}(x,a,b)\ge Q_{h}^{*,\nu^{k}}(x,a,b)+\Pr_{h}\left(\Vover_{h+1}^{k}-V_{h+1}^{*,\nu^{k}}\right)(x,a,b),
\end{align*}
where we recall that $\rho_{h}^{k}(x,a,b):=\beta\left\Vert \phi(x,a,b)\right\Vert _{(\Lambda_{h}^{k})^{-1}}$.
Under the induction hypothesis, we obtain 
\[
\left\langle \phi(x,a,b),\wover_{h}^{k}\right\rangle +\rho_{h}^{k}(x,a,b)\ge Q_{h}^{*,\nu^{k}}(x,a,b)-2(H-h+1)\epsilon\ge 0.
\]
We can now lower-bound $\Qover_{h}^{k}(x,a,b)$:
\begin{align*}
&\Qover_{h}^{k}(x,a,b) \\
& =\projH\left\{ \left\langle \phi(x,a,b),\wover_{h}^{k}\right\rangle +\rho_{h}^{k}(x,a,b)\right\}  &  & \text{by construction}\\
 & \ge\projH\left\{ Q_{h}^{*,\nu^{k}}(x,a,b)-2(H-h+1)\epsilon\right\}  &  & u\ge v\implies\max\left\{ \min\left\{ u,H\right\} ,-H\right\} \ge\max\left\{ \min\left\{ v,H\right\} ,-H\right\} \\
 & \ge\projH\left\{ Q_{h}^{*,\nu^{k}}(x,a,b)\right\} -2(H-h+1)\epsilon &  & \projH\text{ is non-expansive}\\
 & =Q_{h}^{*,\nu^{k}}(x,a,b)-2(H-h+1)\epsilon. &  & Q_{h}^{*,\nu^{k}}(x,a,b)\in[-H,H]
\end{align*}
This proves the inequality (c) for step $h$.

Finally, recall that $\nu_{h}^{k}(x):=\mP_{2}\sigma_{h}^{k}(x)$,
and let $\br(\nu_{h}^{k}(x))$ denote the best response to $\nu_{h}^{k}(x)$
with respect to $Q_{h}^{*,\nu^{k}}(x,\cdot,\cdot)$; i.e., 
\[
\br(\nu_{h}^{k}(x)):=\arg\max_{A\in\simplex}\E_{a\sim A,b\sim\nu_{h}^{k}(x)}\left[Q_{h}^{*,\nu^{k}}(x,a,b)\right].
\]
We then have for all $x$:
\begin{align*}
\Vover_{h}^{k}(x) & :=\E_{(a,b)\sim\sigma_{h}^{k}(x)}\left[\Qover_{h}^{k}(x,a,b)\right] &  & \text{by construction}\\
 & \ge\E_{a'\sim \br(\nu_{h}^{k}(x)),b\sim\mP_{2}\sigma_{h}^{k}(x)}\left[\Qover_{h}^{k}(x,a',b)\right]-2\epsilon &  & \text{\ensuremath{\sigma_{h}^{k}(x)} is \ensuremath{2\epsilon}-CCE by Lemma\,\ref{lem:eps_cce}}\\
 & \ge\E_{a'\sim \br(\nu_{h}^{k}(x)),b\sim\mP_{2}\sigma_{h}^{k}(x)}\left[Q_{h}^{*,\nu^{k}}(x,a',b)\right]-2(H-h+1)\epsilon-2\epsilon &  & \text{inequality (c) we just proved}\\
 & =\E_{a\sim \br(\nu_{h}^{k}(x)),b\sim\nu_{h}^{k}(x)}\left[Q_{h}^{*,\nu^{k}}(x,a,b)\right]-2(H-h+2)\epsilon &  & \text{definition of \ensuremath{\pi_{h}^{k}(x)} and \ensuremath{\nu_{h}^{k}(x)}}\\
 & =V_{h}^{*,\nu^{k}}(x)-2(H-h+2)\epsilon. &  & 
\end{align*}
This proves inequality (iii) for step $h$.
\end{proof}

\subsection{Recursive Decomposition of Duality Gap}\label{sec:proof_main_recursive}

Thanks to Lemma~\ref{lem:ucb_simu} established above, the difference of the UCB and LCB, namely $\delta_{h}^{k}:=\Vover_{h}^{k}(x_{h}^{k})-\Vunder_{h}^{k}(x_{h}^{k})$, is an (approximate) upper bound on the duality gap $ V_{h}^{*,\nu^{k}}(x_{h}^{k})-V_{h}^{\pi^{k},*}(x_{h}^{k}) $. Setting the stage for bounding the duality gap,  we show below that  $ \delta_{h}^{k} $ can be decomposed recursively into the sum of $ \delta_{h+1}^{k} $ and some error terms.
\begin{lem}[Recursive Decomposition]
\label{lem:resursive_simu} Define the random variables
\begin{align*}
\delta_{h}^{k} & :=\Vover_{h}^{k}(x_{h}^{k})-\Vunder_{h}^{k}(x_{h}^{k}),\\
\zeta_{h}^{k} & :=\E\left[\delta_{h+1}^{k}\mid x_{h}^{k},a_{h}^{k},b_{h}^{k}\right]-\delta_{h+1}^{k},\\
\epsover_{h}^{k} & :=\E_{(a,b)\sim\sigma_{h}^{k}(x_{h}^{k})}\left[\Qover_{h}^{k}(x_{h}^{k},a_{h}^{k},b)\right]-\Qover_{h}^{k}(x_{h}^{k},a_{h}^{k},b_{h}^{k}),\\
\epsunder_{h}^{k} & :=\E_{(a,b)\sim\sigma_{h}^{k}(x_{h}^{k})}\left[\Qunder_{h}^{k}(x_{h}^{k},a,b_{h}^{k})\right]-\Qunder_{h}^{k}(x_{h}^{k},a_{h}^{k},b_{h}^{k}).
\end{align*}
Then on the event $\mathfrak{E}$ in Lemma \ref{lem:concentration_simu},
we have for all $(k,h)$, 
\begin{align*}
\delta_{h}^{k} & \le\delta_{h+1}^{k}+\zeta_{h}^{k}+\epsover_{h}^{k}-\epsunder_{h}^{k}+4\beta\sqrt{(\phi_{h}^{k})^{\top}(\Lambda_{h}^{k})^{-1}\phi_{h}^{k}}.
\end{align*}
\end{lem}
\begin{proof}
For each $(x,a,b,k,h)$, by construction we have 
\begin{align*}
\Qover_{h}^{k}(x,a,b)-\Qunder_{h}^{k}(x,a,b) & =\left[(\wover_{h}^{k})^{\top}\phi(x,a,b)+\beta\left\Vert \phi(x,a,b)\right\Vert _{(\Lambda_{h}^{k})^{-1}}\right]-\left[(\wunder_{h}^{k})^{\top}\phi(x,a,b)-\beta\left\Vert \phi(x,a,b)\right\Vert _{(\Lambda_{h}^{k})^{-1}}\right]\\
 & =\left(\wover_{h}^{k}-\wunder_{h}^{k}\right)^{\top}\phi(x,a,b)+2\beta\left\Vert \phi(x,a,b)\right\Vert _{(\Lambda_{h}^{k})^{-1}}.
\end{align*}
The inequalities~(\ref{eq:wover}) and~(\ref{eq:wunder}) in
Lemma~\ref{lem:ls_error_simu} ensure that 
\[
\left(\wover_{h}^{k}-\wunder_{h}^{k}\right)^{\top}\phi(x,a,b)\le\Pr_{h}\left(\Vover_{h+1}^{k}-\Vunder_{h+1}^{k}\right)(x,a,b)+2\beta\left\Vert \phi(x,a,b)\right\Vert _{(\Lambda_{h}^{k})^{-1}},
\]
hence by plugging back we obtain the bound
\begin{equation}
\Qover_{h}^{k}(x,a,b)-\Qunder_{h}^{k}(x,a,b)\le\Pr_{h}\left(\Vover_{h+1}^{k}-\Vunder_{h+1}^{k}\right)(x,a,b)+4\beta\left\Vert \phi(x,a,b)\right\Vert _{(\Lambda_{h}^{k})^{-1}}.\label{eq:Q_bound_simu}
\end{equation}

On the other hand, observe that by definition,
\begin{align*}
\delta_{h}^{k} & :=\Vover_{h}^{k}(x_{h}^{k})-\Vunder_{h}^{k}(x_{h}^{k})\\
 & =\E_{(a,b)\sim\sigma_{h}^{k}(x_{h}^{k})}\left[\Qover_{h}^{k}(x_{h}^{k},a,b)\right]-\E_{(a,b)\sim\sigma_{h}^{k}(x_{h}^{k})}\left[\Qunder_{h}^{k}(x_{h}^{k},a,b)\right]\\
 & =\Qover_{h}^{k}(x_{h}^{k},a_{h}^{k},b_{h}^{k})-\Qunder_{h}^{k}(x_{h}^{k},a_{h}^{k},b_{h}^{k})\\
 & \ensuremath{\qquad}+\left(\E_{(a,b)\sim\sigma_{h}^{k}(x_{h}^{k})}\left[\Qover_{h}^{k}(x_{h}^{k},a,b)\right]-\Qover_{h}^{k}(x_{h}^{k},a_{h}^{k},b_{h}^{k})\right)-\left(\E_{(a,b)\sim\sigma_{h}^{k}(x_{h}^{k})}\Qunder_{h}^{k}\left[(x_{h}^{k},a,b)\right]-\Qunder_{h}^{k}(x_{h}^{k},a_{h}^{k},b_{h}^{k})\right)\\
 & =\Qover_{h}^{k}(x_{h}^{k},a_{h}^{k},b_{h}^{k})-\Qunder_{h}^{k}(x_{h}^{k},a_{h}^{k},b_{h}^{k})+\epsover_{h}^{k}-\epsunder_{h}^{k}.
\end{align*}
Applying the inequality (\ref{eq:Q_bound_simu}), we obtain
\begin{align*}
\delta_{h}^{k} & \le\Pr_{h}\left(\Vover_{h+1}^{k}-\Vunder_{h+1}^{k}\right)(x_{h}^{k},a_{h}^{k},b_{h}^{k})+4\beta\left\Vert \phi(x_{h}^{k},a_{h}^{k})\right\Vert _{(\Lambda_{h}^{k})^{-1}}+\epsover_{h}^{k}-\epsunder_{h}^{k}\\
 & =\E\left[\delta_{h+1}^{k}\mid x_{h}^{k},a_{h}^{k},b_{h}^{k}\right]+4\beta\left\Vert \phi_{h}^{k}\right\Vert _{(\Lambda_{h}^{k})^{-1}}+\epsover_{h}^{k}-\epsunder_{h}^{k}\\
 & =\delta_{h+1}^{k}+\zeta_{h}^{k}+4\beta\left\Vert \phi_{h}^{k}\right\Vert _{(\Lambda_{h}^{k})^{-1}}+\epsover_{h}^{k}-\epsunder_{h}^{k}
\end{align*}
as desired.
\end{proof}

\subsection{Establishing Duality Gap Bound}\label{sec:proof_main_together}

We are now ready to prove Theorem \ref{thm:main_simu}. First observe
that on the event $\mathfrak{E}$ in Lemma \ref{lem:concentration_simu}
(which holds with probability at least $1-p/2$), we have for all $ k\in[K] $:
\begin{align*}
V_{1}^{*,\nu^{k}}(x_{1}^{k})-V_{1}^{\pi^{k},*}(x_{1}^{k})
& \le \Vover_{1}^{k}(x_{1}^{k})-\Vunder_{1}^{k}(x_{1}^{k})+8H\epsilon &  & \text{Lemma \ref{lem:ucb_simu}}\\
&\le \Vover_{1}^{k}(x_{1}^{k})-\Vunder_{1}^{k}(x_{1}^{k})+\frac{8}{K}. & & \text{by the choice $\epsilon=\frac{1}{KH}$}
\end{align*}
This proves the first inequality~\eqref{eq:main_simu1} in Theorem~\ref{thm:main_simu}. 

We next bound the cumulated difference between the UCB and LCB that appear in the RHS of the last inequality. We have 
\begin{align*}
\sum_{k=1}^{K}\left[\Vover_{1}^{k}(x_{1}^{k})-\Vunder_{1}^{k}(x_{1}^{k})\right]  & =\sum_{k=1}^{K}\delta_{1}^{k} &  & \text{definition of $ \delta_{1}^{k} $}\\
 & \le\sum_{k=1}^{K}\sum_{h=1}^{H}(\zeta_{h}^{k}+\epsover_{h}^{k}-\epsunder_{h}^{k})+4\beta\sum_{k=1}^{K}\sum_{h=1}^{H}\sqrt{(\phi_{h}^{k})^{\top}(\Lambda_{h}^{k})^{-1}\phi_{h}^{k}}. &  & \text{Lemma \ref{lem:resursive_simu}}
\end{align*}
We bound the first two RHS terms separately.
\begin{itemize}
\item For the first term, we know that $(\zeta_{h}^{k}+\epsover_{h}^{k}-\epsunder_{h}^{k})$
is a martingale difference sequence (with respect to both $h$ and
$k$), and $\left|\zeta_{h}^{k}+\epsover_{h}^{k}-\epsunder_{h}^{k}\right|\le6H$.
Hence by Azuma-Hoeffding, we have with probability at least $1-p/2$,
\[
\sum_{k=1}^{K}\sum_{h=1}^{H}(\zeta_{h}^{k}+\epsover_{h}^{k}-\epsunder_{h}^{k})\lesssim H\cdot\sqrt{KH\iota}.
\]
\item For the second term, we apply the Elliptical Potential Lemma \ref{lem:elliptic_potential}
to obtain
\begin{align*}
\sum_{h=1}^{H}\sum_{k=1}^{K}\sqrt{(\phi_{h}^{k})^{\top}(\Lambda_{h}^{k})^{-1}\phi_{h}^{k}} & \le\sum_{h=1}^{H}\sqrt{K}\sqrt{\sum_{k=1}^{K}(\phi_{h}^{k})^{\top}(\Lambda_{h}^{k})^{-1}\phi_{h}^{k}} &  & \text{Jensen's inequality}\\
 & \le\sum_{h=1}^{H}\sqrt{K}\cdot\sqrt{2\log\left(\frac{\det\Lambda_{h}^{K}}{\det\Lambda_{h}^{0}}\right)} &  & \text{Lemma \ref{lem:elliptic_potential}}\\
 & \le\sum_{h=1}^{H}\sqrt{K}\cdot\sqrt{2\log\left(\frac{(\lambda+K\max_{k}\left\Vert \phi_{h}^{k}\right\Vert ^{2})^{d}}{\lambda^{d}}\right)} &  & \text{by construction of \ensuremath{\Lambda_{h}^{k}}}\\
 & \le\sum_{h=1}^{H}\sqrt{K}\cdot\sqrt{2d\log\left(\frac{\lambda+K}{\lambda}\right)} &  & \left\Vert \phi_{h}^{k}\right\Vert \le1,\forall h,k\text{ by assumption}\\
 & \le H\sqrt{2Kd\iota}.
\end{align*}
\end{itemize}
Combining the above inequalities, we obtain that with probability at least $1-p/2$, 
\[
\sum_{k=1}^{K}\left[\Vover_{1}^{k}(x_{1}^{k})-\Vunder_{1}^{k}(x_{1}^{k})\right]
\lesssim H\sqrt{HK\iota}+4\beta\cdot H\sqrt{2Kd\iota}
\lesssim\sqrt{d^{3}H^{3}T\iota^{2}},
\]
by our choice of $\beta\asymp dH\sqrt{\iota}$ and the fact that $T=KH$.
This proves the second inequality~\eqref{eq:main_simu2} in Theorem \ref{thm:main_simu}.

Finally, recalling the definition of $ \text{Gap}(K) $ and combining the inequalities~\eqref{eq:main_simu1} and \eqref{eq:main_simu2} we just proved, we obtain that with probability at least $ 1-p $,
\begin{align*}
\text{Gap}(K) & :=\sum_{k=1}^{K}\left[V_{1}^{*,\nu^{k}}(x_{1}^{k})-V_{1}^{\pi^{k},*}(x_{1}^{k})\right] &  & \\
& \le\sum_{k=1}^{K}\left[\Vover_{1}^{k}(x_{1}^{k})-\Vunder_{1}^{k}(x_{1}^{k})\right]+8
\lesssim\sqrt{d^{3}H^{3}T\iota^{2}},
\end{align*}
thereby proving the third inequality~\eqref{eq:main_simu3} in Theorem \ref{thm:main_simu}.

%% file: conclusion.tex
\section{Conclusion}\label{sec:conclusion}

In this paper, we develop provably efficient reinforcement learning methods for zero-sum Markov Games with simultaneous moves and a linear structure. To ensure efficient exploration, our algorithms construct appropriate UCB/LCB for both players and make crucial use of the concept of Coarse Correlated Equilibrium. We provide regret bounds under both the offline and online settings. Corollaries of these bounds apply to turn-based games and the tabular settings. Our results build on and generalize work on learning MDPs with linear structures, and at the same time highlight the crucial differences and new challenges in the game setting. 

A number of directions are of interest for future research. An immediate step is to investigate whether the dependence on the dimension $ d $ and horizon $ H $ in our bounds can be improved and what are the optimal scaling. It would also be interesting to improve our online regret bounds to exploit a weak opponent, in the sense that we can compete with the best response to the opponent, not just competing with the NE. Generalizations to general-sum Markov games, as well as to games with more complicated, nonlinear structures, are also of great interest. %Finally, it is an intriguing question to us whether the use of $\epsilon$-net in the algorithm is necessary or can be avoided by a more refined analysis.

%% file: algos_turn.tex
\section{Algorithms and Proofs for Turn-based Games}

In this section, we present our algorithms for turn-based games and prove the performance guarantees in Corollaries~\ref{thm:main_turn} and~\ref{thm:main_online_turn}.

\subsection{Offline Setting} \label{sec:algos_turn_offline}

In this, the algorithm for turn-based games is given in Algorithm~\ref{alg:turn}, which is derived by specializing the corresponding simultaneous-move Algorithm~\ref{alg:simu} to the turn-based setting.

\begin{algorithm}[tbh]
	\caption{\omvi\ (Turn-Based, Offline)\label{alg:turn}}
	
	\begin{algorithmic}[1]
		
	\State \textbf{Input:} bonus parameter $ \beta>0 $.
	
	\For{ episode $k=1,2,\ldots,K$ }
	
	\State Receive initial state $x_{1}^{k}$.
	
    \For{ step $h=H,H-1,\ldots,2,1$ } \Comment{update policy}
	
	\State $\Lambda_{h}^{k}\leftarrow\sum_{\tau=1}^{k-1}\phi(x_{h}^{\tau},a_{h}^{\tau})\phi(x_{h}^{\tau},a_{h}^{\tau})^{\top}+I.$
	
	\State $\wover_{h}^{k}\leftarrow(\Lambda_{h}^{k})^{-1}\sum_{\tau=1}^{k-1}\phi(x_{h}^{\tau},a_{h}^{\tau})\left[r_{h}(x_{h}^{\tau},a_{h}^{\tau})+\Vover_{h+1}^{k}(x_{h+1}^{\tau})\right]$.
	
	\State $\wunder_{h}^{k}\leftarrow(\Lambda_{h}^{k})^{-1}\sum_{\tau=1}^{k-1}\phi(x_{h}^{\tau},a_{h}^{\tau})\left[r_{h}(x_{h}^{\tau},a_{h}^{\tau})+\Vunder_{h+1}^{k}(x_{h+1}^{\tau})\right]$.
	
	\State $\Qover_{h}^{k}(\cdot,\cdot)\leftarrow\projH\left\{ (\wover_{h}^{k})^{\top}\phi(\cdot,\cdot)+\beta\sqrt{\phi(\cdot,\cdot)^{\top}(\Lambda_{h}^{k})^{-1}\phi(\cdot,\cdot)}\right\} $
	
	\State $\Qunder_{h}^{k}(\cdot,\cdot)\leftarrow\projH\left\{ (\wunder_{h}^{k})^{\top}\phi(\cdot,\cdot)-\beta\sqrt{\phi(\cdot,\cdot)^{\top}(\Lambda_{h}^{k})^{-1}\phi(\cdot,\cdot)}\right\} $
	
	\State Let 
	\[
	\begin{cases}
	\pi_{h}^{k}(\cdot)\leftarrow\findmax\left(\Qover_{h}^{k},\cdot\right),\Vover_{h}^{k}(\cdot)\leftarrow\Qover_{h}^{k}\left(\cdot,\pi_{h}^{k}(\cdot)\right),\Vunder_{h}^{k}(\cdot)\leftarrow\Qunder_{h}^{k}\left(\cdot,\pi_{h}^{k}(\cdot)\right) & I(\cdot)=1\\
	\nu_{h}^{k}(\cdot)\leftarrow\findmin\left(\Qunder_{h}^{k},\cdot\right),\Vover_{h}^{k}(\cdot)\leftarrow\Qover_{h}^{k}\left(\cdot,\nu_{h}^{k}(\cdot)\right),\Vunder_{h}^{k}(\cdot)\leftarrow\Qunder_{h}^{k}\left(\cdot,\nu_{h}^{k}(\cdot)\right) & I(\cdot)=2
	\end{cases}
	\]
	
	\EndFor
	
	\For {step $h=1,2,\ldots,H$ } \Comment{execute policy}
	
	\State \textbf{if} $I(x_{h}^{k})=1$, $\pone$ takes action
	$a_{h}^{k}=\pi_{h}^{k}(x_{h}^{k})$,
	
	\State \textbf{else} if $I(x_{h}^{k})=2$, $\ptwo$ takes
	action $a_{h}^{k}=\nu_{h}^{k}(x_{h}^{k})$.
	
	\State Observe next state $x_{h+1}^{k}$.
	
	\EndFor
	
	\EndFor
	\end{algorithmic}
\end{algorithm}

The algorithm involves the subroutines $\findmax$ and $\findmin$,
which are derived by specializing the $\findcce$ routine in Algorithm~\ref{alg:find_cce}
to the turn-based setting. For completeness we provide below a description
of these two subroutines. Let $\mQ$ be the class of functions $Q:\mS\times\mA\to\R$
with the parametric form
\[
Q(x,a)=\left\langle w,\phi(x,a)\right\rangle +\rho\beta\sqrt{\phi(x,a)^{\top}A\phi(x,a)},
\]
where the parameter $(w,A,\rho)$ satisfy $\left\Vert w\right\Vert \le2H\sqrt{dk}$,
$\left\Vert A\right\Vert _{F}\le\beta^{2}\sqrt{d}$ and $\rho\in\{\pm1\}$.
Let $\mQ_{\epsilon}$ be a fixed $\epsilon$-covering of $\mQ$ with
respect to the $\ell_{\infty}$ norm. With these notations, the subroutine
$\findmax$ is given in Algorithm~\ref{alg:find_max}, and the subroutine
$\findmin$ is given by $\findmin(Q,x)=\findmax(-Q,x)$.

\begin{algorithm}[tbh]
	\caption{$\protect\findmax$\label{alg:find_max}}
	
	\begin{algorithmic}[1]
	\State \textbf{Input:} $Q$, $x$ and discretization parameter $\epsilon>0$.
	
	\State Pick $\Qotilde\in\mQ_{\epsilon}$ satisfying $\left\Vert \Qotilde-Q\right\Vert _{\infty}\le\epsilon$.
	
	\State For the input $x$, let $\widetilde{a}=\arg\max_{a}\Qotilde(x,a)$.
	
	\State \textbf{Output:} $\widetilde{a}$.
	\end{algorithmic}

\end{algorithm}

Informally, one may simply think of $\findmax(Q,x)$ as $\arg\max_{a}Q(x,a)$
and $\findmin(Q,x)$ as $\arg\min_{a}Q(x,a)$. As in the simultaneous
move setting, these subroutines are introduced for the technical considerations explained in Section~\ref{sec:tech_consideration}.

\subsubsection{Proof of Corollary~\ref{thm:main_turn}\label{sec:proof_main_turn}}

%\paragraph{Proof of Corollary~\ref{thm:main_turn}:}

We prove Corollary~\ref{thm:main_turn} by specializing Theorem~\ref{thm:main_simu}
to the turn-based setting. Specifically, as argued in Section~\ref{sec:setup_turn},
linear turn-based game is a special case of linear simultaneous games
with 
\begin{equation}
\begin{aligned}\phi(x,a,b) & \equiv\phi(x,a),\quad r_{h}(x,a,b)\equiv r(x,a),\quad\Pr_{h}(x,a,b)\equiv\Pr_{h}(x,a),\qquad\text{ if \ensuremath{x\in\mS_{1}}},\\
\phi(x,a,b) & \equiv\phi(x,b),\quad r_{h}(x,a,b)\equiv r(x,b),\quad\Pr_{h}(x,a,b)\equiv\Pr_{h}(x,b),\qquad\text{ if \ensuremath{x\in\mS_{2}}}.
\end{aligned}
\label{eq:degeneration}
\end{equation}
Moreover, Algorithm~\ref{alg:simu}, when applied to the turn-based
setting, degenerates to Algorithm~\ref{alg:turn}. To see this, note
that under the degeneration of $\phi(x,a,b)$ in (\ref{eq:degeneration}),
the values $\Qover_{h}^{k}$ and $\Qunder_{h}^{k}$ computed in Algorithm~\ref{alg:simu}
only depend on the action of the active player; that is,
\begin{equation}
\begin{aligned}\Qover_{h}^{k}(x,a,b) & \equiv\Qover_{h}^{k}(x,a),\qquad\text{ if \ensuremath{x\in\mS_{1}}},\\
\Qunder_{h}^{k}(x,a,b) & \equiv\Qunder_{h}^{k}(x,b),\qquad\text{ if \ensuremath{x\in\mS_{2}}}.
\end{aligned}
\label{eq:degeneration_Q}
\end{equation}
In this case, one can verify that finding the CCE (cf.\ equation~\eqref{eq:cce})
as done in $\findcce$ degenerates to a unilateral maximization or
minimization problem, namely $\arg\max_{a}\Qotilde(x,a)$ or $\arg\min_{a}\Qotilde(x,a)$.
This is exactly what the subroutines $\findmax$ and $\findmin$ compute.
With the above reduction, Corollary~\ref{thm:main_turn} follows directly
from Theorem~\ref{thm:main_simu}.

\subsection{Online Setting} \label{sec:algos_turn_online}

In this setting, the algorithm for turn-based games is given in  Algorithm~\ref{alg:online_turn}, which is  derived by specializing the corresponding simultaneous-move Algorithm~\ref{alg:online_simu} to the turn-based setting.

\begin{algorithm}[tbh]
	\caption{\omvi\ (Turn-Based, Online)\label{alg:online_turn}}
	
	\begin{algorithmic}[1]
		
	\State \textbf{Input:} bonus parameter $ \beta>0 $.
		
	\For{ episode $k=1,2,\ldots,K$ }
	
	\State Receive initial state $x_{1}^{k}$.
	
	\For{ step $h=H,H-1,\ldots,2,1$ } \Comment{update policy}
	
	\State $\Lambda_{h}^{k}\leftarrow\sum_{\tau=1}^{k-1}\phi(x_{h}^{\tau},a_{h}^{\tau})\phi(x_{h}^{\tau},a_{h}^{\tau})^{\top}+I.$
	
	\State $w_{h}^{k}\leftarrow(\Lambda_{h}^{k})^{-1}\sum_{\tau=1}^{k-1}\phi(x_{h}^{\tau},a_{h}^{\tau})\left[r_{h}(x_{h}^{\tau},a_{h}^{\tau})+V_{h+1}^{k}(x_{h+1}^{\tau})\right]$.
	
	\State $Q_{h}^{k}(\cdot,\cdot)\leftarrow\projH\left\{ (w_{h}^{k})^{\top}\phi(\cdot,\cdot)+\beta\sqrt{\phi(\cdot,\cdot)^{\top}(\Lambda_{h}^{k})^{-1}\phi(\cdot,\cdot)}\right\} $.
	
	\State $V_{h}^{k}(\cdot)\leftarrow\begin{cases}
	\max_{a}Q_{h+1}^{k}(\cdot,a) & \text{if }I(\cdot)=1,\\
	\min_{a}Q_{h+1}^{k}(\cdot,a) & \text{if }I(\cdot)=2.
	\end{cases}$
	\EndFor
	
	\For{ step $h=1,2,\ldots,H$ } \Comment{execute policy}
	
	\State \textbf{if} $I(x_{h}^{k})=1$, take action $a_{h}^{k}=\arg\max_{a}Q_{h}^{k}(x_{h}^{k},a)$,
	
	\State \textbf{else} do nothing and let $\ptwo$ play.
	
	\State Observe next state $x_{h+1}^{k}$.
	
	\EndFor
	
	\EndFor
	
	\end{algorithmic}

\end{algorithm}

\subsubsection{Proof of Corollary~\ref{thm:main_online_turn}\label{sec:proof_main_online_turn}}

%\paragraph{Proof of Corollary~\ref{thm:main_online_turn}:}

We prove Corollary~\ref{thm:main_online_turn} by specializing
Theorem~\ref{thm:main_online_simu} to the turn-based setting. The
argument is essentially the same as that in the proof of Corollary~\ref{thm:main_turn}
above. We omit the details. 

%% file: techlemma.tex
%!TEX root =main.tex

\section{Technical Lemmas\label{sec:tech_lemma}}

The proofs of our main Theorems~\ref{thm:main_simu} and~\ref{thm:main_online_simu}
involve several common steps. We summarize these steps as several
lemmas, which are either proved below or are standard in the literature.

\subsection{Boundedness of Linear Coefficients}

We begin with two simple lemmas about boundedness of the linear coefficients
of $Q$ functions.
\begin{lem}[True Coefficients Are Bounded]
\label{lem:true_bounded}Under Assumption \ref{assu:linear_bounded_simu},
for each policy pair $(\pi,\nu)$ of $\pone$ and $\ptwo$, the linear
coefficient of their action-value function $Q_{h}^{\pi,\nu}(x,a,b)=\left\langle \phi(x,a,b),w_{h}^{\pi,\nu}\right\rangle $
satisfies
\[
\left\Vert w_{h}^{\pi,\nu}\right\Vert \le2H\sqrt{d},\qquad\forall h\in[H].
\]
\end{lem}
\begin{proof}
From the Bellman equation, we have
\begin{align*}
\phi(x,a,b)^{\top}w_{h}^{\pi,\nu}=Q_{h}^{\pi,\nu}(x,a,b) & =r_{h}(x,a,b)+(\Pr_{h}V_{h+1}^{\pi,\nu})(x,a,b)\\
 & =\phi(x,a,b)^{\top}\theta_{h}+\int V_{h+1}^{\pi,\nu}(x')\phi(x,a,b)^{\top}d\mu_{h}(x'),\quad\forall x,a,b,h.
\end{align*}
Assuming that $\left\{ \phi(x,a,b)\right\} $ spans $\R^{d}$ and
solving the linear equation, we obtain
\[
w_{h}^{\pi,\nu}=\theta_{h}+\int V_{h+1}^{\pi,\nu}(x')d\mu_{h}(x').
\]
Under the normalization Assumption \ref{assu:linear_bounded_simu},
we have $\left\Vert \theta_{h}\right\Vert \le\sqrt{d}$, $\left\Vert \mu_{h}(\mathcal{S})\right\Vert \le\sqrt{d}$
and $\left|V_{h+1}^{\pi,\nu}(x')\right|\le H$. It follows that 
\[
\left\Vert w_{h}^{\pi,\nu}\right\Vert \le\sqrt{d}+H\sqrt{d}\le2H\sqrt{d}
\]
as desired.
\end{proof}

An immediate consequence of the above lemma is that $\left\Vert w_{h}^{\pi,*}\right\Vert \le2H\sqrt{d}$
and $\left\Vert w_{h}^{*,\nu}\right\Vert \le2H\sqrt{d}$; cf.~Remark~\ref{rem:linear_best_response}.
\begin{lem}[Algorithm Coefficients Are Bounded]
\label{lem:alg_bounded}The coefficients $\{\wover_{h}^{k},\text{\ensuremath{\wunder}}_{h}^{k}\}$
in Algorithm~\ref{alg:simu} and the coefficients $\{w_{h}^{k}\}$ in
Algorithm~\ref{alg:online_simu} satisfy
\[
\left\Vert \wover_{h}^{k}\right\Vert \le2H\sqrt{dk},\quad\left\Vert \wunder_{h}^{k}\right\Vert \le2H\sqrt{dk},\quad\text{and}\quad\left\Vert w_{h}^{k}\right\Vert \le2H\sqrt{dk},\qquad\forall (k,h)\in[K]\times [H].
\]
\end{lem}
\begin{proof}
We only prove the last inequality. The other two inequalities can
be established in exactly the same way. For each $k$ and $h$, we
have 
\begin{align*}
\left\Vert w_{h}^{k}\right\Vert  & =\left\Vert \left(\Lambda_{h}^{k}\right)^{-1}\sum_{\tau=1}^{k-1}\phi(x_{h}^{\tau},a_{h}^{\tau},b_{h}^{\tau})\left[r_{h}(x_{h}^{\tau},a_{h}^{\tau},b_{h}^{\tau})+V_{h+1}^{k}(x_{t+1}^{\tau})\right]\right\Vert \\
 & \le\sum_{\tau=1}^{k-1}\left\Vert \left(\Lambda_{h}^{k}\right)^{-1}\phi(x_{h}^{\tau},a_{h}^{\tau},b_{h}^{\tau})\right\Vert \cdot2H &  & \left|r_{h}\right|\le H,\left|V_{h+1}^{k}\right|\le H\\
 & \le\sum_{\tau=1}^{k-1}\left\Vert \left(\Lambda_{h}^{k}\right)^{-1/2}\right\Vert \cdot\left\Vert \phi(x_{h}^{\tau},a_{h}^{\tau},b_{h}^{\tau})\right\Vert _{(\Lambda_{h}^{k})^{-1}}\cdot2H\\
 & \le\sqrt{k\sum_{\tau=1}^{k-1}\left\Vert \phi(x_{h}^{\tau},a_{h}^{\tau},b_{h}^{\tau})\right\Vert _{(\Lambda_{h}^{k})^{-1}}^{2}}\cdot2H &  & \Lambda_{h}^{k}\succeq I\text{ and Jensen's}\\
 & \le\sqrt{kd}\cdot2H, &  & \text{Lemma \ref{lem:simple_bound}}
\end{align*}
thereby proving the last inequality in the lemma.
\end{proof}

\subsection{Inequalities for Summations}

We next state two lemmas for summations. The first lemma is from \citet[Lemma D.1]{jin2019linear}.

\begin{lem}[Simple Upper Bound]
\label{lem:simple_bound}If $\Lambda_{t}=\lambda I+\sum_{i\in[t]}\phi_{i}\phi_{i}^{\top}$, where $\phi_i\in\R^{d}$
and $\lambda>0$, then
\[
\sum_{i\in[t]}\phi_{i}^{\top}\Lambda_{t}^{-1}\phi_{i}\le d.
\]
\end{lem}

The second lemma can be found in {\citet[Lemma 11]{abbasi2011improved}} and {\citet[Lemma D.2]{jin2019linear}}.
\begin{lem}[Elliptical Potential Lemma]
\label{lem:elliptic_potential}Suppose  that $\{\phi_t\}_t\geq 0 $ is a sequence in $\R^d$ satisfying $\left\Vert \phi_{t}\right\Vert \le1,\forall t$. Let 
$\Lambda_{0}\in\R^{d\times d}$ be a positive definite matrix, and
$\Lambda_{t}=\Lambda_{0}+\sum_{i\in[t]}\phi_{i}\phi_{i}^{\top}$.
If the smallest eigenvalue of $\Lambda_{0}$ satisfies $\lambda_{\min}(\Lambda_{0})\ge1$,
then 
\[
\log\left(\frac{\det\Lambda_{t}}{\det\Lambda_{0}}\right)\le\sum_{j\in[t]}\phi_{j}^{\top}\Lambda_{j-1}^{-1}\phi_{j}\le2\log\left(\frac{\det\Lambda_{t}}{\det\Lambda_{0}}\right),\forall t.
\]
\end{lem}

\subsection{Covering and Concentration Inequalities for Self-normalized Processes}

The first lemma below is useful for establishing uniform concentration.
Recall the function class $\mQ$ defined in the text around equation~(\ref{eq:Qclass}).
\begin{lem}[Covering]
\label{lem:covering_Q}The $\epsilon$-covering number of $\mQ$
with respect to the $\ell_{\infty}$ norm satisfies
\[
\mathcal{N}_{\epsilon}\le2\left(1+\frac{8H\sqrt{dk}}{\epsilon}\right)^{d}\left(1+\frac{8\beta^{2}\sqrt{d}}{\epsilon^{2}}\right)^{d^{2}}.
\]
\end{lem}
\begin{proof}
For any two functions $Q,Q'\in\mQ$ with parameters $(w,A,\rho)$
and $(w',A',\rho)$, we have
\begin{align*}
 & \left\Vert Q-Q'\right\Vert _{\infty}\\
 & =\sup_{x,a,b}\left|\projH\left\{ \left\langle w,\phi(x,a,b)\right\rangle +\rho\beta\sqrt{\phi(x,a,b)^{\top}A\phi(x,a,b)}\right\} -\projH\left\{ \left\langle w',\phi(x,a,b)\right\rangle -\rho\beta\sqrt{\phi(x,a,b)^{\top}A'\phi(x,a,b)}\right\} \right|\\
 & \le\sup_{\phi:\left\Vert \phi\right\Vert \le1}\left|\left\langle w-w',\phi\right\rangle +\rho\beta\sqrt{\phi^{\top}A\phi}-\rho\beta\sqrt{\phi^{\top}A'\phi}\right|\\
 & \le\sup_{\phi:\left\Vert \phi\right\Vert \le1}\left|\left\langle w-w',\phi\right\rangle \right|+\sup_{\phi:\left\Vert \phi\right\Vert \le1}\sqrt{\left|\phi^{\top}(A-A')\phi\right|}\\
 & \le\left\Vert w-w'\right\Vert +\sqrt{\left\Vert A-A'\right\Vert _{F}},
\end{align*}
where the second last inequality follows due to the fact that  $|\sqrt{x}-\sqrt{y}|\leq \sqrt{|x-y|}$ holds for any $x,y\geq 0.$

Therefore, a 0-cover $\mathcal{C}_{\rho}$ of $\left\{ \pm1\right\} $,
an $\epsilon/2$-cover $\mathcal{C}_{w}$ of $\left\{ w\in \R^d:\left\Vert w\right\Vert \le2H\sqrt{dk}\right\} $
and an $\epsilon^{2}/4$-cover $\mathcal{C}_{A}$ of $\left\{ A\in \R^{d\times d}:\left\Vert A\right\Vert _{F}\le\beta^{2}\sqrt{d}\right\} $
implies an $\epsilon$-cover of $\mQ$. It follows that 
\[
\mathcal{N}_{\epsilon}\le\left|\mathcal{C}_{\rho}\right|\left|\mathcal{C}_{w}\right|\left|\mathcal{C}_{A}\right|\le2\left(1+\frac{8H\sqrt{dk}}{\epsilon}\right)^{d}\left(1+\frac{8\beta^{2}\sqrt{d}}{\epsilon^{2}}\right)^{d^{2}},
\]
where the last step follows from standard bounds on the covering number
of Euclidean Balls, e.g., \citet[Lemma 5.2]{vershynin2010nonasym}.
\end{proof}

The next lemma, originally from~\citet[Theorem 1]{abbasi2011improved}, is now standard in the bandit literature.

\begin{lem}[Concentration for Self-normalized Processes]
\label{lem:self_normalized}Suppose $\{\epsilon_{t}\}_{t\geq 1}$ is a scalar
stochastic process generating the filtration $\{\mathcal{F}_{t}\}_{t \geq 0}$, and $\epsilon_{t}|\mathcal{F}_{t-1}$ is zero mean and $\sigma$-subGaussian.
Let $\{\phi_{t}\}_{t \geq 1 }$ be an $\R^{d}$-valued stochastic process with
$\phi_{t}\in\mathcal{F}_{t-1}$. Suppose $\Lambda_{0}\in\R^{d\times d}$
is positive definite, and $\Lambda_{t}=\Lambda_{0}+\sum_{s=1}^{t}\phi_{s}\phi_{s}^{\top}$.
Then for each $\delta\in (0,1),$ with probability at least $1-\delta$, we have 
\[
\left\Vert \sum_{s=1}^{t}\phi_{s}\epsilon_{s}\right\Vert _{\Lambda_{t}^{-1}}^{2}\le2\sigma^{2}\log\left[\frac{\det(\Lambda_{t})^{1/2}\det(\Lambda_{0})^{-1/2}}{\delta}\right],\qquad\forall t\ge0.
\]
\end{lem}

%% file: appendix.tex
%!TEX root =main.tex 

\section{Proof of Theorem \ref{thm:main_online_simu}\label{sec:proof_main_online_simu}}

In this section, we prove Theorem~\ref{thm:main_online_simu} for the online
setting of simultaneous games. We shall make use of the technical
lemmas given in Appendix~\ref{sec:tech_lemma}. Recall the shorthand
$\phi_{h}^{k}:=\phi(x_{h}^{k},a_{h}^{k},b_{h}^{k})$. The proof follows a similar strategy as that for the proof of Theorem~\ref{thm:main_simu} in Section~\ref{sec:proof_main_simu}. In particular, our proof consists of five steps as  presented in the subsections to follow.

\subsection{Uniform Concentration}

In the online setting, the value function estimate $V_{h+1}^{k}(x)$
is computed using the NE of the \emph{zero-sum }game defined by a
\emph{single }payoff matrix $Q_{h+1}^{k}(x,\cdot,\cdot)$. It is easier
to establish uniform concentration in this setting. To see why, we
recall the function class $\mQ$ defined in the text around equation~(\ref{eq:Qclass}),
and introduce the related function class
\[
\mV:=\left\{ V:\mS\to\R,V(x)=\max_{A\in \Delta}\min_{B \in \Delta}\E_{a\in A,b\in B}Q(x,a,b),Q\in\mQ\right\} .
\]
In words, $\mV$ contains the possible values of the NEs of the zero-sum matrix games
in $\mQ$. As we show in the lemma below, an $\epsilon$-cover of
the set $\mQ$ immediately induces an $\epsilon$-cover of the set
$\mV$, thanks to the non-expansiveness of the maximin operator for
zero-sum games. (Note that general-sum games and their CCEs do not
have such a non-expansiveness property in general; see Appendix~\ref{sec:stability} for details.)
\begin{lem}[Covering]
	\label{lem:covering_V}The $\epsilon$-covering number of $\mV$
	with respect to the $\ell_{\infty}$ norm is upper bounded by 
	\[
	\mathcal{N}_{\epsilon}\le2\left(1+\frac{8H\sqrt{dk}}{\epsilon}\right)^{d}\left(1+\frac{8\beta^{2}\sqrt{d}}{\epsilon^{2}}\right)^{d^{2}}.
	\]
\end{lem}
\begin{proof}
	For any two functions $V,V'\in\mV$, let them take the form $V(\cdot)=\max_{A \in \Delta }\min_{B\in \Delta}\E_{a\in A,b\in B}Q(\cdot,a,b)$
	and $V'(\cdot)=\max_{A \in \Delta}\min_{B \in \Delta}\E_{a\in A,b\in B}Q'(\cdot,a,b)$ with
	$Q,Q'\in\mQ$. Since the maximin operator is non-expansive, we have
	\begin{align*}
	\left\Vert V-V'\right\Vert _{\infty} & =\sup_{x}\left|\max_{A \in \Delta}\min_{B \in \Delta}\E_{a\in A,b\in B}Q(\cdot,a,b)-\max_{A\in \Delta}\min_{B\in \Delta}\E_{a\in A,b\in B}Q'(\cdot,a,b)\right|\\
	& \le\sup_{x,a,b}\left|Q(x,a,b)-Q'(x,a,b)\right|\\
	& =\left\Vert Q-Q'\right\Vert _{\infty}.
	\end{align*}
	Therefore, an $\epsilon$-cover of $\mQ$ induces an $\epsilon$-cover
	of $\mV$, and hence the $\epsilon$-covering number of $\mV$ is
	upper bounded by the $\epsilon$-covering number of $\mQ$. Recalling
	that the latter number is bounded in Lemma~\ref{lem:covering_Q},
	we complete the proof of the desired bound.
\end{proof}
\begin{lem}[Concentration]
	\label{lem:concentration_online_simu} Under the setting of Theorem~\ref{thm:main_online_simu},
	for each $p\in(0,1)$, the following event $\mathfrak{E}$ holds with
	probability at least $1-p/2$:
	\begin{align*}
	\left\Vert \sum_{\tau\in[k-1]}\phi_{h}^{\tau}\left[V_{h+1}^{k}(x_{h+1}^{\tau})-\left(\Pr_{h}V_{h+1}^{k}\right)(x_{h}^{\tau},a_{h}^{\tau},b_{h}^{\tau})\right]\right\Vert _{(\Lambda_{h}^{k})^{-1}} & \lesssim dH\sqrt{\log(dT/p)},\qquad\forall (k,h)\in [K]\times [H].
	\end{align*}
\end{lem}
\begin{proof}
	Fix $ (k,h) \in [K] \times [H] $. Define the filtration $\{\mathcal{F}_{\tau} \}$ as in equation~\eqref{eq:filtration}.%Let $\mathcal{F}_{\tau}:=\mathcal{F}(x_{\cdot}^{1},a_{\cdot}^{1},\ldots,x_{\cdot}^{\tau},a_{\cdot}^{\tau})$ be the $ \sigma $-algebra generated by the data from the first $ \tau $ episodes. 
	%Note that $\phi_{\cdot}^{\tau},a_{\cdot}^{\tau}\in\mathcal{F}_{\tau-1}$.
	
	Set $\epsilon=\frac{1}{K}$ and let $\mV_{\epsilon}$ be a minimal
	$\epsilon$-net of $\mV$. Fix a function $\Votilde\in\mV_{\epsilon}.$
	The random variable $\Votilde(x_{h+1}^{\tau})-\Pr_{h}\Votilde(x_{h}^{\tau})$, when conditioned on $ \mathcal{F}_{\tau-1} $,
	is zero-mean and $2H$-bounded. Applying Lemma~\ref{lem:self_normalized}
	gives
	\[
	\left\Vert \sum_{\tau\in[k-1]}\phi_{h}^{\tau}\left(\Votilde(x_{h+1}^{\tau})-\Pr_{h}\Votilde(x_{h}^{\tau},a_{h}^{\tau},b_{h}^{\tau})\right)\right\Vert _{(\Lambda_{h}^{k})^{-1}}\lesssim dH\sqrt{\log(dT/p)}
	\]
	with probability at least $2^{-\Omega(d^{2}\log(dT/p))}$. Now note
	that $\left|\mV_{\epsilon}\right|=\mathcal{N}_{\epsilon}\le2\left(1+\frac{8H\sqrt{dk}}{\epsilon}\right)^{d}\left(1+\frac{\beta^{2}\sqrt{d}}{\epsilon^{2}}\right)^{d^{2}}$
	by Lemma~\ref{lem:covering_V}. By a union bound, the above inequality
	holds for all $\Votilde\in\mV_{\epsilon}$ with probability at least
	$1-p/2$.
	
	Now, for each $V_{h+1}^{k}\in\mV$ (the inclusion follows from Lemma~\ref{lem:alg_bounded}),
	let $\Votilde\in\mV_{\epsilon}$ be the closest point in the net.
	The difference $\Delta=V_{h+1}^{k}-\Votilde$ satisfies $\left\Vert \Delta\right\Vert _{\infty}\le\epsilon$.
	It follows that 
	\begin{align*}
	& \left\Vert \sum_{\tau\in[k-1]}\phi_{h}^{\tau}\left[V_{h+1}^{k}(x_{h+1}^{\tau})-\left(\Pr_{h}V_{h+1}^{k}\right)(x_{h}^{\tau},a_{h}^{\tau},b_{h}^{\tau})\right]\right\Vert _{(\Lambda_{h}^{k})^{-1}}\\
	\le & \left\Vert \sum_{\tau\in[k-1]}\phi_{h}^{\tau}\left[\Votilde(x_{h+1}^{\tau})-\left(\Pr_{h}\Votilde\right)(x_{h}^{\tau},a_{h}^{\tau},b_{h}^{\tau})\right]\right\Vert _{(\Lambda_{h}^{k})^{-1}}+\left\Vert \sum_{\tau\in[k-1]}\phi_{h}^{\tau}\left[\Delta(x_{h+1}^{\tau})-\left(\Pr_{h}\Delta\right)(x_{h}^{\tau},a_{h}^{\tau},b_{h}^{\tau})\right]\right\Vert _{(\Lambda_{h}^{k})^{-1}}\\
	\lesssim & dH\sqrt{\log(dT/p)}+\epsilon\sum_{\tau\in[k-1]}\left\Vert \phi_{h}^{\tau}\right\Vert _{(\Lambda_{h}^{k})^{-1}}\\
	\le & dH\sqrt{\log(dT/p)}+\frac{1}{K}\cdot k,
	\end{align*}
	where the last step follows from $\epsilon=\frac{1}{K}$, $\Lambda_{h}^{k}\succeq I$
	and $\left\Vert \phi_{h}^{\tau}\right\Vert \le1$. This completes
	the proof of the lemma.
\end{proof}

\subsection{Least-squares Estimation Error}

Here we bound the difference between the algorithm's value function
(without bonus) and the true value function of any policy $\pi$,
recursively in terms of the step $h$.
\begin{lem}[Least-squares Error Bound] The quantities $\{w_h^k,V_h^k\}$ in Algorithm~\ref{alg:online_simu} satisfy the following.
	If $\beta=dH\sqrt{\iota},$ then on the event $\mathfrak{E}$ in
	Lemma \ref{lem:concentration_online_simu}, we have for all $(x,a,b,h,k)$
	and any policy pair $(\pi,\nu)$:
	\begin{equation}
	\left|\left\langle \phi(x,a,b),w_{h}^{k}\right\rangle -Q_{h}^{\pi,\nu}(x,a,b)-\Pr_{h}(V_{h+1}^{k}-V_{h+1}^{\pi,\nu})(x,a,b)\right|\le\rho_{h}^{k}(x,a,b),\label{eq:w}
	\end{equation}
	where $\rho_{h}^{k}(x,a,b):=\beta\sqrt{\phi(x,a,b)^{\top}\left(\Lambda_{h}^{k}\right)^{-1}\phi(x,a,b)}.$
\end{lem}
\begin{proof}
	The proof is essentially identical to that of Lemma~\ref{lem:ls_error_simu},
	except that we use the concentration result in Lemma~\ref{lem:concentration_online_simu}
	instead of Lemma~\ref{lem:concentration_simu}.
\end{proof}

\subsection{Upper Confidence Bounds}

Here we establish the desired UCB property.
\begin{lem}[UCB]
	\label{lem:ucb_online_simu}On the event $\mathfrak{E}$ in Lemma~\ref{lem:concentration_simu},
	we have for all $(x,a,b,k,h)$:
	\[
	Q_{h}^{k}(x,a,b)\ge Q_{h}^{*}(x,a,b),\qquad V_{h}^{k}(x)\ge V_{h}^{*}(x).
	\]
\end{lem}
\begin{proof}
	We fix $k$ and perform induction on $h$. The base case $h=H$ holds
	since the terminal cost is zero. Now assume that the bounds hold for
	step $h+1$; that is, $Q_{h+1}^{k}(x,a,b)\ge Q_{h+1}^{*}(x,a,b)$
	and $V_{h+1}^{k}(x)\ge V_{h+1}^{*}(x),\forall(x,a,b).$ By construction
	we have
	\begin{align*}
	Q_{h}^{k}(x,a,b) & =\projH\left\{ \left\langle \phi(x,a,b),w_{h}^{k}\right\rangle +\beta\left\Vert \phi(x,a,b)\right\Vert _{(\Lambda_{h}^{k})^{-1}}\right\} .
	\end{align*}
	On the other hand, note that $Q_{h}^{*}=Q_{h}^{\pi^{*},\nu^{*}}$
	and $V_{h}^{*}=V_{h}^{\pi^{*},\nu^{*}}$, hence by inequality~(\ref{eq:w})
	in Lemma~\ref{lem:ls_error_simu} applied to $(\pi,\nu)=(\pi^{*},\nu^{*})$,
	we have 
	\[
	\left|\left\langle \phi(x,a,b),w_{h}^{k}\right\rangle -Q_{h}^{*}(x,a,b)-\Pr_{h}(V_{h+1}^{k}-V_{h+1}^{*})(x,a,b)\right|\le\beta\left\Vert \phi(x,a,b)\right\Vert _{(\Lambda_{h}^{k})^{-1}}.
	\]
	Plugging back we obtain
	\[
	Q_{h}^{k}(x,a,b)\ge\projH\left\{ Q_{h}^{*}(x,a,b)+\Pr_{h}(V_{h+1}^{k}-V_{h+1}^{*})(x,a,b)\right\} .
	\]
	Under the induction hypothesis, we have $V_{h+1}^{k}(x)-V_{h+1}^{*}(x)\ge0$
	for each $x\in\mS$, whence 
	\[
	Q_{h}^{k}(x,a,b)\ge\projH\left\{ Q_{h}^{*}(x,a,b)\right\} =Q_{h}^{*}(x,a,b).
	\]
	Consequently, we have 
	\begin{align*}
	V_{h}^{k}(x) & =\max_{A\in\simplex}\min_{B\in\simplex}\E_{a\sim A,b\sim B}\left[Q_{h}^{k}(x,a,b)\right] &  & \text{algorithm specification}\\
	& \ge\max_{A\in\simplex}\min_{B\in\simplex}\E_{a\sim A,b\sim B}\left[Q_{h}^{*}(x,a,b)\right]\\
	& =V_{h}^{*}(x). &  & \text{definition}
	\end{align*}
	We conclude that the bounds hold for step $h$.
\end{proof}

\subsection{Recursive Regret Decomposition}

Thanks to Lemma~\ref{lem:ucb_online_simu}, the regret $V_{1}^{*}(x_{1}^{k})-V_{1}^{\pi^{k},\nu^{k}}(x_{1}^{k})$ of interest is upper bounded by the difference  $V_{1}^{k}(x_{1}^{k})-V_{1}^{\pi^{k},\nu^{k}}(x_{1}^{k})$
 between the empirical value (with bonus) and true value
of the agent's policy $\pi^{k}$. We next derive a recursive (in $ h $) formula for this difference.
\begin{lem}[Recursive Decomposition]
	\label{lem:resursive_online_simu} Define the random variables
	\begin{align*}
	\delta_{h}^{k} & :=V_{h}^{k}(x_{h}^{k})-V_{h}^{\pi^{k},\nu^{k}}(x_{h}^{k}),\\
	\zeta_{h}^{k} & :=\E\left[\delta_{h+1}^{k}\mid x_{h}^{k},a_{h}^{k},b_{h}^{k}\right]-\delta_{h+1}^{k},\\
	\eps_{h}^{k} & :=\E_{a\sim\pi_{h}^{k}(x_{h}^{k})}\left[Q_{h}^{k}(x_{h}^{k},a,b_{h}^{k})\right]-Q_{h}^{k}(x_{h}^{k},a_{h}^{k},b_{h}^{k}),\\
	\widehat{\eps}_{h}^{k} & :=\E_{a\sim\pi^{k}(x_{h}^{k}),b\sim\nu_{h}^{k}(x_{h}^{k})}\left[Q_{h}^{\pi^{k},\nu^{k}}(x_{h}^{k},a,b)\right]-Q_{h}^{\pi^{k},\nu^{k}}(x_{h}^{k},a_{h}^{k},b_{h}^{k}).
	\end{align*}
	Then on the event $\mathfrak{E}$ in Lemma~\ref{lem:concentration_simu},
	we have for all $(k,h)$:
	\[
	\delta_{h}^{k}\le\delta_{h+1}^{k}+\zeta_{h}^{k}+\eps_{h}^{k}-\widehat{\eps}_{h}^{k}+2\beta\sqrt{(\phi_{h}^{k})^{\top}(\Lambda_{h}^{k})^{-1}\phi_{h}^{k}}.
	\]
\end{lem}
\begin{proof}
	By algorithm specification and the fact that $(\pi_{h}^{k}(x_{h}^{k}),B_{0})$
	is the NE of $Q_{h}^{k}(x_{h}^{k},\cdot,\cdot)$, we
	have 
	\begin{align*}
	V_{h}^{k}(x_{h}^{k}) & =\min_{b}\E_{a\sim\pi_{h}^{k}(x_{h}^{k})}\left[Q_{h}^{k}(x_{h}^{k},a,b)\right]\\
	& \le\E_{a\sim\pi_{h}^{k}(x_{h}^{k})}\left[Q_{h}^{k}(x_{h}^{k},a,b_{h}^{k})\right]\\
	& =Q_{h}^{k}(x_{h}^{k},a_{h}^{k},b_{h}^{k})+\eps_{h}^{k},
	\end{align*}
	and by definition we have 
	\begin{align*}
	V_{h}^{\pi^{k},\nu^{k}}(x_{h}^{k}) & =\E_{a\sim\pi^{k}(x_{h}^{k}),b\sim\nu_{h}^{k}(x_{h}^{k})}\left[Q_{h}^{\pi^{k},\nu^{k}}(x_{h}^{k},a,b)\right]\\
	& =Q_{h}^{\pi^{k},\nu^{k}}(x_{h}^{k},a_{h}^{k},b_{h}^{k})+\widehat{\eps}_{h}^{k}.
	\end{align*}
	It follows that 
	\[
	\delta_{h}^{k}\le Q_{h}^{k}(x_{h}^{k},a_{h}^{k},b_{h}^{k})-Q_{h}^{\pi^{k},\nu^{k}}(x_{h}^{k},a_{h}^{k},b_{h}^{k})+\eps_{h}^{k}-\widehat{\eps}_{h}^{k}.
	\]
	On the other hand, by construction of $Q_{h}^{k}$ and Lemma~\ref{eq:ls_error_simu},
	we have for all $(x,a,b)$, 
	\[
	Q_{h}^{k}(x,a,b)-Q_{h}^{\pi^{k},\nu^{k}}(x,a,b)\le\Pr_{h}(V_{h+1}^{k}-V_{h+1}^{\pi^{k},\nu^{k}})(x,a,b)+2\beta\sqrt{\phi(x,a,b)^{\top}\left(\Lambda_{h}^{k}\right)^{-1}\phi(x,a,b)}.
	\]
	Combining pieces, we obtain that
	\begin{align*}
	\delta_{h}^{k} & \le\Pr_{h}(V_{h+1}^{k}-V_{h+1}^{\pi^{k},\nu^{k}})(x_{h}^{k},a_{h}^{k},b_{h}^{k})+\eps_{h}^{k}-\widehat{\eps}_{h}^{k}+2\beta\sqrt{(\phi_{h}^{k})^{\top}(\Lambda_{h}^{k})^{-1}\phi_{h}^{k}}\\
	& =\E\left[\delta_{h+1}^{k}\mid x_{h}^{k},a_{h}^{k},b_{h}^{k}\right]+\eps_{h}^{k}-\widehat{\eps}_{h}^{k}+2\beta\sqrt{(\phi_{h}^{k})^{\top}(\Lambda_{h}^{k})^{-1}\phi_{h}^{k}}\\
	& =\delta_{h+1}^{k}+\zeta_{h}^{k}+\eps_{h}^{k}-\widehat{\eps}_{h}^{k}+2\beta\sqrt{(\phi_{h}^{k})^{\top}(\Lambda_{h}^{k})^{-1}\phi_{h}^{k}}
	\end{align*}
	as desired.
\end{proof}

\subsection{Establishing Regret Bound}

We are now ready to prove Theorem \ref{thm:main_online_simu}.
First observe that 
\begin{align*}
\text{Regret}(K) & :=\sum_{k=1}^{K}\left[V_{1}^{*}(x_{1}^{k})-V_{1}^{\pi^{k},\nu^{k}}(x_{1}^{k})\right] &  & \text{definition}\\
& \le\sum_{k=1}^{K}\left[V_{1}^{k}(x_{1}^{k})-V_{1}^{\pi^{k},\nu^{k}}(x_{1}^{k})\right] &  & V_{1}^{k}(x_{1}^{k})\ge V_{h}^{*}(x_{1}^{k})\text{ by Lemma \ref{lem:ucb_online_simu}}\\
& =\sum_{k=1}^{K}\delta_{1}^{k} &  & \text{definition}\\
& \le\sum_{k=1}^{K}\sum_{h=1}^{H}(\zeta_{h}^{k}+\eps_{h}^{k}-\widehat{\eps}_{h}^{k})+2\beta\sum_{k=1}^{K}\sum_{h=1}^{H}\sqrt{(\phi_{h}^{k})^{\top}(\Lambda_{h}^{k})^{-1}\phi_{h}^{k}}. &  & \text{Lemma \ref{lem:resursive_online_simu}}
\end{align*}
We bound the two RHS terms separately.
\begin{itemize}
	\item For the first term, we know that $(\zeta_{h}^{k}+\eps_{h}^{k}-\widehat{\eps}_{h}^{k})$
	is a martingale difference sequence (with respect to both $h$ and
	$k$), and $\left|\zeta_{h}^{k}+\eps_{h}^{k}-\widehat{\eps}_{h}^{k}\right|\le6H$.
	Hence by Azuma-Hoeffding, we have w.h.p.
	\[
	\sum_{k=1}^{K}\sum_{h=1}^{H}(\zeta_{h}^{k}+\eps_{h}^{k}-\widehat{\eps}_{h}^{k})\lesssim H\cdot\sqrt{KH\iota}=H\sqrt{T\iota}.
	\]
	\item For the second term, we apply the Elliptical Potential Lemma~\ref{lem:elliptic_potential}
	to obtain
	\begin{align*}
	\sum_{h=1}^{H}\sum_{k=1}^{K}\sqrt{(\phi_{h}^{k})^{\top}(\Lambda_{h}^{k})^{-1}\phi_{h}^{k}} & \le\sum_{h=1}^{H}\sqrt{K}\sqrt{\sum_{k=1}^{K}(\phi_{h}^{k})^{\top}(\Lambda_{h}^{k})^{-1}\phi_{h}^{k}} &  & \text{Jensen's inequality}\\
	& \le\sum_{h=1}^{H}\sqrt{K}\cdot\sqrt{2\log\left(\frac{\det\Lambda_{h}^{K}}{\det\Lambda_{h}^{0}}\right)} &  & \text{Lemma \ref{lem:elliptic_potential}}\\
	& \le\sum_{h=1}^{H}\sqrt{K}\cdot\sqrt{2\log\left(\frac{(\lambda+K\max_{k}\left\Vert \phi_{h}^{k}\right\Vert ^{2})^{d}}{\lambda^{d}}\right)} &  & \text{by construction of \ensuremath{\Lambda_{h}^{k}}}\\
	& \le\sum_{h=1}^{H}\sqrt{K}\cdot\sqrt{2d\log\left(\frac{\lambda+K}{\lambda}\right)} &  & \left\Vert \phi_{h}^{k}\right\Vert \le1,\forall h,k\text{ by assumption}\\
	& \le H\sqrt{2Kd\iota}.
	\end{align*}
\end{itemize}
Combining, we obtain that
\[
\text{Regret}(K)\lesssim H\sqrt{T\iota}+\beta\cdot H\sqrt{2Kd\iota}\lesssim\sqrt{d^{3}H^{3}T\iota^{2}}
\]
by our choice of $\beta\asymp dH\sqrt{\iota}.$ This completes the
proof of Theorem \ref{thm:main_online_simu}.

\section{Efficient Implementation of $\protect\findcce$\label{sec:find_cce_implement}}

The main computation step in $\findcce$ involves finding an element in the
fixed $\epsilon$-cover $\mQ_{\epsilon}$ that is close to a given
function $Q \in \mQ$. Here we discuss how to efficiently implement this procedure
without explicitly maintaining the cover $\mQ_{\epsilon}$. 

Recall that each element in the function class $\mQ$ is defined by a pair $ (w,A) \in \R^d \times \R^{d\times d} $; see equation~\eqref{eq:Qclass}. Therefore, the cover
$\mQ_{\epsilon}$ is induced, up to scaling, by an $\epsilon$-cover $\mathcal{C}_{w}$
in $\ell_{2}$ norm of the Euclidean ball $\mathbb{B}_{w} := \big\{ w\in\R^{d}:\left\Vert w\right\Vert \le1\big\} $ as well as an $\epsilon^2$-cover $\mathcal{C}_{A}$ of
the ball $\mathbb{B}_{A} := \big\{ A\in\R^{d\times d}:\left\Vert A\right\Vert _{F}\le1\big\} $; cf.\ the proof of Lemma~\ref{lem:covering_Q}.
We may replace $\mathcal{C}_{w}$ by a cover $\mathcal{C}_{w,\infty}$ in the $\ell_{\infty}$
norm; similarly for $\mathcal{C}_{A}$. Clearly, an $ \ell_\infty $ cover is also an $\ell_{2}$ cover; moreover, using an $ \ell_\infty $ cover allows for
efficient computation of near neighbors by simple rounding. The price we pay is an additional
dimension factor $d$ in the covering number, which eventually goes
into the log term.

We now provide the details for covering  $\mathbb{B}_{w}$; the idea applies similarly to covering $\mathbb{B}_{A}$. 
\begin{lem}\label{lem:fast_covering}
	Let $ \epsilon>0 $ be a given accuracy parameter. There exists a set $\mathcal{C}_{w,\infty}  $ satisfying the following: 
	%(i) $ \mathcal{C}_{w,\infty} $ is an $ \frac{\epsilon}{\sqrt{d}} $-cover in $ \ell_{\infty} $ norm of $ \mathbb{B}_{w} $ and hence an $ \epsilon $-cover in $ \ell_{2} $ norm of $ \mathbb{B}_{w} $;
	 (i) $ \log\left|\mathcal{C}_{w,\infty}\right| \le d\log\left(1+\frac{2\sqrt{d}}{\epsilon}\right)$; (ii) for each vector $ w\in \mathbb{B}_{w} $, we can find,  in $ O(d) $ time, a vector $\widetilde{w}\in \mathcal{C}_{w,\infty}$ that satisfies $\left\Vert \widetilde{w}-w\right\Vert_\infty \le \frac{\epsilon}{\sqrt{d}}$ and hence $\left\Vert \widetilde{w}-w\right\Vert \le\epsilon$.
\end{lem}

\begin{proof}
Set $\epsilon_{0}:=\frac{\epsilon}{\sqrt{d}}$. We discretize the
interval $G:=[-1,1]$ into an $\epsilon_{0}$-grid as
\[
G_{\epsilon_{0}}:=\left\{ k\epsilon_{0}:k=-\left\lfloor \frac{1}{\epsilon_{0}}\right\rfloor ,-\left\lfloor \frac{1}{\epsilon_{0}}\right\rfloor +1,\ldots,-2,-1,0,1,2,\ldots,\left\lfloor \frac{1}{\epsilon_{0}}\right\rfloor -1,\left\lfloor \frac{1}{\epsilon_{0}}\right\rfloor \right\} .
\]
We then let $\mathcal{C}_{w,\infty}:=(G_{\epsilon_{0}})^{d}$.
The log cardinality of $\mathcal{C}_{w,\infty}$ is 
\[
\log\left|\mathcal{C}_{w,\infty}\right|=\log\left|G_{\epsilon_{0}}\right|^{d}=\log\left(1+2\left\lfloor \frac{1}{\epsilon_{0}}\right\rfloor \right)^{d}\le d\log\left(1+\frac{2\sqrt{d}}{\epsilon}\right),
\]
as claimed in part (i) of the lemma.
Compare this bound with the log cardinality of the optimal $\epsilon$-cover in $\ell_{2}$  norm of $\big\{ w\in\R^{d}:\left\Vert w\right\Vert \le1\big\} $:
$
\log |\mathcal{C}_{w}|\asymp d\log\left(1+\frac{2}{\epsilon}\right).
$
We see that the former is only logarithmic larger than the latter.

Moreover, for each vector $w$ in the ball $\big\{ w'\in\R^{d}:\left\Vert w'\right\Vert \le1\big\} $, we can efficiently find a vector $\widetilde{w}\in \mathcal{C}_{w,\infty}$
that satisfies $\left\Vert \widetilde{w}-w\right\Vert_\infty \le \frac{\epsilon}{\sqrt{d}}$ and hence  $\left\Vert \widetilde{w}-w\right\Vert \le \epsilon$.
To do this, we simply let
\[
\widetilde{w}_{i}=\left\lfloor \frac{\left|w_{i}\right|}{\epsilon_{0}}\right\rfloor \cdot\epsilon_{0}\cdot\sign(w_{i}),\qquad\text{for each \ensuremath{i\in[d]}},
\]
with the convention that $ \sign(0) = 0 $. Note that $ \widetilde{w} $  can be computed in $ O(d) $ time. Moreover, since $\left\Vert w\right\Vert \le1 $,  for each $i\in[d]$ we have $ \left|w_{i}\right|\le1 $ and hence
\[
\left\lfloor \frac{\left|w_{i}\right|}{\epsilon_{0}}\right\rfloor \in\left\{ 0,1,\ldots,\left\lfloor \frac{1}{\epsilon_{0}}\right\rfloor \right\} ,
\]
which means $\widetilde{w}_{i}\in G_{\epsilon_{0}}$. It follows that  
$\widetilde{w}\in(G_{\epsilon_{0}})^{d}=\mathcal{C}_{w,\infty}$ as claimed.
Finally,  the approximation accuracy satisfies  
\begin{align*}
\left\Vert \widetilde{w}-w\right\Vert_\infty & =\max_{i\in[d]}\left|\left\lfloor \frac{\left|w_{i}\right|}{\epsilon_{0}}\right\rfloor \cdot\epsilon_{0}\cdot\sign(w_{i})-w_{i}\right|\\
 & =\epsilon_{0}\max_{i\in[d]}\left|\left\lfloor \frac{\left|w_{i}\right|}{\epsilon_{0}}\right\rfloor \cdot\sign(w_{i})-\frac{\left|w_{i}\right|}{\epsilon_{0}}\cdot\sign(w_{i})\right| &  & w_{i}=\left|w_{i}\right|\cdot\sign(w_{i})\\
 & \le\epsilon_{0}\max_{i\in[d]}1\cdot\left|\sign(w_{i})\right| &  & \left|\left\lfloor x\right\rfloor -x\right|\le1\\
 & \le  \epsilon_0 = \frac{\epsilon}{\sqrt{d}}. &  & %\epsilon_{0}:=\frac{\epsilon}{\sqrt{d}}\text{ by choice}
\end{align*}
This proves part (ii) of the lemma.
\end{proof}

\section{Instability of the Value of General-Sum Game\label{sec:stability}}

In the analysis of our algorithms (in particular, in proving uniform concentration in the proof of Theorem~\ref{thm:main_simu}), we encounter the following question:
Is the value of the CCE of a general-sum game stable under
perturbation to the payoff matrices? Here we show that the answer is negative in general, by demonstrating a counter example. 

Consider a two-player general-sum matrix game, and recall our convention that player~1 tries to maximize and player~2 tries to minimize (cf.\ Section~\ref{sec:notation}). Let $u_{i}:\mA\times\mA\to\R$
be the payoff matrix of player $i\in\{1,2\}$, such that player $i$
receives the payoff $u_{i}(a,b)$ when players 1 and 2 take actions $a$
and $b$, respectively. Let $\sigma\in\simplex(\mA\times\mA)$ be
any notion of CCE that is unique; e.g.,
the social-optimal or max-entropy CCE. In this equilibrium, the expected payoff of
player $i$ is 
\[
V_{i}(u_{1},u_{2}):=\E_{(a,b)\sim\sigma}\left[u_{i}(a,b)\right].
\]
We say that the game value $V=(V_{1},V_{2})$ is a Lipschitz function of the payoff matrices $u=(u_{1},u_{2})$ if there exists
a universal constant $C$  such that 
\[
\underbrace{ \max_{i\in\{1,2\}}\left|V_{i}(u_{1},u_{2})-V_{i}(u_{1}',u_{2}')\right| }_{\|V(u)-V(u')\|_\infty}
	\le C\cdot \underbrace{ \max_{j\in\{1,2\}}\max_{a,b\in\mA}\left|u_{j}(a,b)-u_{j}'(a,b)\right| }_{\| u - u'\|_\infty}, \quad \forall u, u'.
\]
The following example shows that $V$ is in general not Lipschitz in $ u $.\footnote{We learned the example from \url{https://mathoverflow.net/questions/347366/perturbation-of-the-value-of-a-general-sum-game-at-a-equilibirium}}
\begin{lem}\label{lem:not_lipschitz}
	For any $ \epsilon>0 $, there exists a pair of games $ u $ and $ u' $, each with a unique CCE, such that 
	\begin{align*}
	 \|u - u'\| \le 2\epsilon \quad\text{and}\quad  \| V(u) - V(u') \|_\infty \ge 1.
	\end{align*}
\end{lem}

\begin{proof}
Consider two games $ u $ and $ u' $ with payoff matrices
\[
(u_{1}, u_{2})=\begin{pmatrix}1+\epsilon,-1-\epsilon & \epsilon,-1\\
1,-\epsilon & 0,0
\end{pmatrix}\qquad\text{and}\qquad(u_{1}', u_{2}')=\begin{pmatrix}1-\epsilon,-1+\epsilon & -\epsilon,-1\\
1,\epsilon & 0,0
\end{pmatrix},
\]
where $\epsilon>0$. Note that the two pairs of payoff matrices satisfy $ \|u - u'\|_\infty=2\epsilon $, so $u$ and $u'$ can be made arbitrarily close. The game $ u $ has a unique CCE, which is the deterministic policy (or pure strategy) corresponding to the top-left entry of the payoff matrices; similarly, the game $ u' $ has a unique CCE corresponding to the bottom-right entry. These two CCEs have  values
\[
\left(V_{1}(u_{1},u_{2}),V_{2}(u_{1},u_{2})\right)=(1+\epsilon,-1-\epsilon)\qquad\text{and}\qquad\left(V_{1}(u_{1}',u_{2}'),V_{2}(u_{1}',u_{2}')\right)=(0,0),
\]
which are bounded away from each other as claimed. 
\end{proof}

We note that in example above,  the CCE policy of the game $ u $ is an $ \epsilon $-approximate CCE of the game $ u' $, and vice versa, as any unilateral deviation leads to at most $ \epsilon $ improvement in the payoff.